\documentclass[twoside,11pt]{article}

\usepackage{jmlr2e}

\usepackage{multirow}

\usepackage{enumitem}
\usepackage{bm}%
\usepackage{amsmath}%
\usepackage{amsfonts}%
\usepackage{amssymb}%
\usepackage{colonequals}
\usepackage{psfrag}
\usepackage{mathrsfs}
\usepackage{xcolor}
\usepackage{hyperref}

\usepackage{booktabs,siunitx}

\usepackage{algorithm}
\usepackage{algpseudocode}

\hypersetup{%
  colorlinks=true,% hyperlinks will be black
  urlcolor=blue,%
  linkcolor=blue,%
  citecolor=blue,%
  linkbordercolor=blue,% hyperlink borders will be red
  pdfborderstyle={/S/U/W 1}% border style will be underline of width 1pt
}

\newcommand{\be}{\bm{e}}
\newcommand{\bx}{\bm{x}}
\newcommand{\by}{\bm{y}}
\newcommand{\bv}{\bm{v}}
\newcommand{\bz}{\bm{z}}
\newcommand{\bw}{\bm{w}}

\DeclareMathOperator*{\argmin}{\arg\!\min}

\newcommand{\Reals}{\mathbb{R}}

\newcommand{\defined}{\triangleq}

\newcommand{\ExpVal}[2]{\mathbb{E}\left[ #2 \right]}

\newcommand{\ba}{\bm{a}}
\newcommand{\bb}{\bm{b}}

\newcommand{\Varr}[2]{\mathsf{Var}_{#1}\left(#2\right)}
\newcommand{\Var}[1]{\mathsf{Var}\left(#1\right)}

\newcommand{\EE}[1]{\ExpVal{}{#1}}

\newcommand{\indicator}[1]{\mathbb{I}_{#1}}

\newcommand{\mmae}{\mathsf{mmae}}
\newcommand{\median}{\mathsf{median}}

\newcommand{\rU}{\textnormal{U}}
\newcommand{\rV}{\textnormal{V}}
\newcommand{\rW}{\textnormal{W}}

\newcommand{\rX}{\textnormal{X}}
\newcommand{\rY}{\textnormal{Y}}
\newcommand{\rZ}{\textnormal{Z}}

\newcommand{\rS}{\textnormal{S}}

\newcommand{\rN}{\textnormal{N}}

\newcommand{\dif}{\textrm{d}}

\newcommand{\TV}{\textnormal{D}_{\scalebox{.6}{\textnormal TV}}}

\newcommand{\chisquare}{\textnormal{D}_{\chi^2}}

\newcommand{\EEE}[2]{\mathbb{E}_{#1}\left[ #2 \right]}

\newcommand{\KL}{\textnormal{D}_{\scalebox{.6}{\textnormal KL}}}

\definecolor{longhorn}{rgb}{0.8, 0.33, 0.0}

\newtheorem{assumption}{Assumption}

\usepackage{lastpage}
\jmlrheading{23}{2022}{1-\pageref{LastPage}}{11/21; Revised
8/22}{11/22}{21-1396}{Hao Wang, Rui Gao, Flavio P. Calmon}

\ShortHeadings{Generalization Bounds for Noisy Iterative Algorithms}{Wang, Gao, and Calmon}

\firstpageno{1}

\begin{document}

\title{Generalization Bounds for Noisy Iterative Algorithms \\Using Properties of Additive Noise Channels}

\author{\name Hao Wang \email {hao\_wang@g.harvard.edu}\\
       \addr Harvard University
       \AND
       \name Rui Gao \email rui.gao@mccombs.utexas.edu \\
       \addr The University of Texas at Austin
       \AND
       \name Flavio P. Calmon \email flavio@seas.harvard.edu\\
       \addr Harvard University
       }

\editor{Gabor Lugosi}

\maketitle

\begin{abstract}%
Machine learning models trained by different optimization algorithms under different data distributions can exhibit distinct generalization behaviors. In this paper, we analyze the generalization of models trained by noisy iterative algorithms. We derive distribution-dependent generalization bounds by connecting noisy iterative algorithms to additive noise channels found in communication and information theory. Our generalization bounds shed light on several applications, including differentially private stochastic gradient descent (DP-SGD), federated learning, and stochastic gradient Langevin dynamics (SGLD). We demonstrate our bounds through numerical experiments, showing that they can help understand recent empirical observations of the generalization phenomena of neural networks.
\end{abstract}

\section{Introduction}

Many learning algorithms aim to solve the following (possibly non-convex) optimization problem:
\begin{align}
    \min_{\bw \in \mathcal{W}}~L_{\mu}(\bw),\quad \text{where } L_{\mu}(\bw) \defined \EE{\ell(\bw, \rZ)} = \int_{\mathcal{Z}}\ell(\bw, \bz) \dif \mu(\bz),
\end{align}
where $\bw\in \mathcal{W} \subseteq \Reals^d$ is the model parameter (e.g., weights of a neural network) to optimize; $\mu$ is the underlying data distribution that generates $\rZ$; and $\ell:\mathcal{W}\times\mathcal{Z}\to \Reals^+$ is the loss function (e.g., 0-1 loss). In the context of supervised learning, $\rZ$ is often composed by a feature vector $\rX$ and its corresponding label $\rY$. Since the data distribution $\mu$ is unknown, $L_{\mu}(\bw)$ cannot be computed directly. In practice, a data set $\rS\defined (\rZ_1,\cdots,\rZ_n)$ containing $n$ i.i.d. points $\rZ_i\sim \mu$ is used to minimize an empirical risk:
\begin{align}
\label{eq::emp_risk_opt}
    \min_{\bw\in\mathcal{W}}~L_{\rS}(\bw), \quad \text{where } L_{\rS}(\bw) \defined \frac{1}{n}\sum_{i=1}^n \ell(\bw,\rZ_i).
\end{align}

We consider the following (projected) noisy iterative algorithm for solving the empirical risk optimization in \eqref{eq::emp_risk_opt}. The parameter $\bw$ is initialized with a random point $\rW_0 \in \mathcal{W}$ and updated using the following rule:
\begin{align}
\label{eq::rec_update_parameter_Extend}
    \rW_t = \mathsf{Proj}_{\mathcal{W}} \left(\rW_{t-1} - \eta_t \cdot g(\rW_{t-1}, \{\rZ_{i}\}_{i \in \mathcal{B}_t}) + m_t \cdot \rN_t \right),
\end{align}
where $\eta_t$ is the learning rate; $\rN_t$ is an additive noise drawn independently from a distribution $P_{\rN}$; $m_t$ is the magnitude of the noise; $\mathcal{B}_t \subseteq [n]$ contains the indices of the data points used at the current iteration and $b_t \defined |\mathcal{B}_t|$; $g$ is the direction for updating the parameter (e.g., gradient of the loss function); and 
\begin{align}
\label{eq::update_minibatch}
    g(\rW_{t-1}, \{\rZ_{i}\}_{i \in \mathcal{B}_t}) \defined \frac{1}{b_t}\sum_{i\in \mathcal{B}_t} g(\rW_{t-1}, \rZ_i).
\end{align}
At the end of each iteration, the parameter is projected onto the domain $\mathcal{W}$, i.e., $\mathsf{Proj}_{\mathcal{W}}(\bw) \defined \argmin_{\bw'\in \mathcal{W}}\|\bw' - \bw\|$. The recursion in \eqref{eq::rec_update_parameter_Extend} is run $T$ iterations and the final output is a random variable $\rW_T$. 

The goal of this paper is to provide an upper bound for the \emph{expected generalization gap}: 
\begin{align}
\label{eq::exp_gen_gap}
    \EE{L_\mu(\rW_T)-L_{\rS}(\rW_T)},
\end{align}
where the expectation is taken over the randomness of the training data set $\rS$ and of the noisy iterative algorithm.

Noisy iterative algorithms are used in different practical settings due to their many attractive properties \citep[see e.g.,][]{li2016preconditioned,zhang2017hitting,raginsky2017non,xu2018global}. 
For example, differentially private SGD (DP-SGD) algorithm \citep[see e.g.,][]{song2013stochastic,abadi2016deep}, one kind of noisy iterative algorithm, is often used to train machine learning models while protecting user privacy \citep{dwork2006calibrating}. Recently, it has been implemented in open-source libraries, including Opacus \citep{FB2020Opacus} and TensorFlow Privacy \citep{G2019TFPrivacy}. 
The additive noise in iterative algorithms may also mitigate overfitting for deep neural networks (DNNs) \citep{neelakantan2015adding}. 
From a theoretical perspective, noisy iterative algorithms can escape local minima \citep{kleinberg2018alternative} or saddle points \citep{ge2015escaping} and generalize well \citep{pensia2018generalization}.

We derive generalization bounds for noisy iterative algorithms. Although these bounds may be vacuous, when calculated numerically, they maintain a high correlation with the generalization gap. As a result, they shed light on many empirical observations of neural networks that are not explained by uniform notions of hypothesis class complexity \citep{vapnik1971uniform,valiant1984theory}.
For example, a neural network trained using true labels exhibits better generalization ability than a network trained using corrupted labels even when the network architecture is fixed and perfect training accuracy is achieved \citep{zhang2016understanding}. Distribution-independent bounds is unable to capture this phenomenon because they are invariant to both true data and corrupted data.\footnote{Note that there are generalization bounds \citep[see e.g.,][]{bartlett1998sample,bartlett2017spectrally,koltchinskii2002empirical} that implicitly depend on data distribution through, e.g., margin and/or weight matrices' norm. Since different training data result in distinct weight matrices, these bounds may capture some generalization phenomena, such as label corruption.} In contrast, our bounds capture this empirical observation, exhibiting a lower value on networks trained on true labels compared to ones trained on corrupted labels (Figure~\ref{Fig::label_corruption}). 
Another example is that a wider network often has a more favourable generalization capability \citep{neyshabur2014search}. This may seem counter-intuitive at first glance since one may expect that wider networks have a higher VC-dimension and, consequently, would have a higher generalize gap. Our bounds capture this behaviour and are decreasing with respect to the neural network width (Figure~\ref{Fig::Hidden_Units}).

We present three generalization bounds for the noisy iterative algorithms (Section~\ref{sec::Gen_Bound}). These bounds rely on different kinds of $f$-divergence but are proved in a uniform manner by exploring properties of additive noise channels (Section~\ref{sec::prop_add_channel}). Among them, the KL-divergence bound can deal with sampling with replacement; the total variation bound is often the tightest one; and the $\chi^2$-divergence bound requires the mildest assumption. We apply our results to applications, including DP-SGD, federated learning, and SGLD (Section~\ref{sec::application}). Under these applications, our generalization bounds can be significantly simplified and estimated from the training data. Finally, we demonstrate our bounds through numerical experiments (Section~\ref{sec::experiments}), showing that they can predict the behavior of the true generalization~gap.

Our generalization bounds incorporate a time-decaying factor. This decay factor tightens the bounds by enabling the impact of early iterations to reduce with time. Our analysis is motivated by a line of recent works \citep{feldman2018privacy,balle2019privacy,asoodeh2020privacy} which observed that data points used in the early iterations enjoy stronger differential privacy guarantees than those occurring late. Accordingly, we prove that if a data point is used at an early iteration, its contribution to our generalization bounds is decreasing with time due to the cumulative effect of the additive noise.

The proof techniques of this paper are based on fundamental tools from information theory. We first use an information-theoretic framework, proposed by \citet{russo2016controlling} and \citet{xu2017information} and further tightened by \citet{bu2020tightening}, for deriving algorithmic generalization bounds. This framework relates the generalization gap in \eqref{eq::exp_gen_gap} with the $f$-information\footnote{The $f$-information (see \eqref{eq::defn_f_inf} for its definition) includes a family of measures, such as mutual information, which quantify the dependence between two random variables.} $I_f(\rW_T;\rZ_i)$ between the algorithmic output $\rW_T$ and each individual data point $\rZ_i$. However, estimating this $f$-information from data is intractable since the underlying distribution is unknown. Given this major challenge, we connect the noisy iterative algorithms with additive noise channels, a fundamental model used in data transmission. As a result, we further upper bound the $f$-information by a quantity that can be estimated from data by developing new properties of additive noise channels. Moreover, we incorporate a time-decaying factor into our bounds. This factor is established by strong data processing inequalities \citep{dobrushin1956central,cohen1998comparisons} and has an intuitive interpretation: the dependence between algorithmic output $\rW_T$ and the data points used in the early iterations is decreasing with time due to external additive noise (i.e., $I_f(\rW_T;\rZ_i)$ is decreasing with $T$ for a fixed $\rZ_i$).

\subsection{Related Works}

There are significant recent works which adopt the information-theoretic framework \citep{xu2017information} for analyzing the generalization capability of noisy iterative algorithms. Among them, \citet{pensia2018generalization} initially derived a generalization bound in Corollary~1 and their bound was extended in Proposition~3 of \citet{bu2020tightening} for the SGLD algorithm. Although the framework in \citet{pensia2018generalization} can be applied to a broad class of noisy iterative algorithms, their bound in Corollary~1 and Proposition~3 in \citet{bu2020tightening} rely on the Lipschitz constant of the loss function, which makes them independent of the data distribution. Distribution-independent bounds can be potentially loose since the Lipschitz constant may be large and may not capture some empirical observations (e.g., label corruption \citep{zhang2016understanding}). Specifically, this Lipschitz constant only relies on the architecture of the network instead of the weight matrices or the data distribution so it is the same for a network trained from corrupted data and a network trained from true data.

To obtain a distribution-dependent bound, \citet{negrea2019information} improved the analysis in \citet{pensia2018generalization} by replacing the Lipschitz constant with a gradient prediction residual when analyzing the SGLD algorithm.
Their follow-up work \citep{haghifam2020sharpened} investigated the Langevin dynamics algorithm (i.e., full batch SGLD), which was later extended by \citet{rodriguez2020random} to SGLD, and observed a time-decaying phenomenon in their experiments. Specifically, \citep{haghifam2020sharpened} incorporated a quantity, namely the squared error probability of the hypothesis test, into their bound in Theorem~4.2 and this quantity decays with the number of iterations. This seems to suggest that earlier iterations have a larger impact on their generalization bound. In contrast, our decay factor indicates that the impact of earlier iterations is reducing with the total number of iterations. Furthermore, the bound in their Theorem~4.2 requires a bounded loss function while our $\chi^2$-based generalization bound only needs the variance of the loss function to be bounded. 
More broadly, \citet{neu2021information} investigated the generalization properties of SGD. However, the generalization bound in their Proposition~3 suffers from a weaker order $O(1/\sqrt{n})$ when the analysis is applied to the SGLD algorithm.

In addition to the works discussed above, there is a line of papers on deriving SGLD generalization bounds \citep{mou2018generalization,li2019generalization} through other proof techniques. 
Among them, \citet{mou2018generalization} introduced two generalization bounds. The first one \citep[Theorem~1 of][]{mou2018generalization}, a stability-based bound, achieves $O(1/n)$ rate in terms of the sample size $n$ but relies on the Lipschitz constant of the loss function which makes it distribution-independent. 
The second one \citep[Theorem~2 of][]{mou2018generalization}, a PAC-Bayes bound, replaces the Lipschitz constant by an expected-squared gradient norm but suffers from a slower rate $O(1/\sqrt{n})$. In contrast, our SGLD bound in Proposition~\ref{prop::gen_bound_SGLD} has order $O(1/n)$ and tightens the expected-squared gradient norm by the variance of gradients. The PAC-Bayes bound in \citet{mou2018generalization} also incorporates an explicit time-decaying factor. However, their analysis seems to heavily rely on the Gaussian noise. In contrast, our generalization bounds include a decay factor for a broad class of noisy iterative algorithms.
A follow-up work by \citet{li2019generalization} combined the algorithmic stability approach with PAC-Bayesian theory and presented a bound which achieves order $O(1/n)$. However, their bound requires the scale of the learning rate to be upper bounded by the inverse Lipschitz constant of the loss function, which could result in negligible learning rates in practice. In contrast, we do not need any assumptions on the learning rate.

A standard approach \citep[see e.g.,][]{he2021tighter} of deriving a generalization bound for the DP-SGD algorithm follows two steps: (i) prove that DP-SGD satisfies the $(\epsilon,\delta)$-DP guarantees \citep{song2013stochastic,wu2017bolt,feldman2018privacy,balle2019privacy,asoodeh2020privacy}; (ii) derive/apply a generalization bound that holds for \emph{any} $(\epsilon,\delta)$-DP algorithm \citep{dwork2015preserving,bassily2021algorithmic,jung2019new}. However, generalization bounds obtained from this procedure are distribution-independent since DP is robust with respect to the data distribution. In contrast, our bounds in Section~\ref{subsec::DP_SGD} are distribution-dependent. We extend our analysis and derive a generalization bound in the setting of federated learning in Section~\ref{subsec::FL}. A previous work by \citet{yagli2020information} also proved a generalization bound for federated learning in their Theorem~3 but their bound involves a mutual information which could be hard to estimate from data.

The conference version \citep{wang21analyzing} of this work investigates the generalization of the SGLD and DP-SGD algorithms. In this paper, we study a broader class of noisy iterative algorithms, including SGLD and DP-SGD as two examples. This extension requires a significant improvement of the proof techniques presented in \citet{wang21analyzing}. Specifically, we introduce a unified framework for deriving generalization bounds through $f$-divergence while the analysis in the prior work \citep{wang21analyzing} is tailored to the KL-divergence. In the context of the DP-SGD algorithm, we prove that the total variation distance leads to a tighter generalization bound and the $\chi^2$-divergence leads to a bound requiring milder assumptions compared with the results in \citet{wang21analyzing}. Finally, we derive a new generalization bound in the context of federated learning as an application of our framework.

\section{Preliminaries}
\label{sec::Prelim}

\subsection{Notations}

For a positive integer $n$, we define the set $[n] \defined \{1,\cdots,n\}$. We denote by $\|\cdot\|_1$ and $\|\cdot\|_2$ the 1-norm and 2-norm of a vector, respectively. A random variable $\rX$ is $\sigma$-sub-Gaussian if $\log\EE{\exp{\lambda(\rX - \EE{\rX})}}\leq \sigma^2 \lambda^2/2$ for any $\lambda\in\Reals$. 
For a random vector $\rX = (\rX_1,\cdots,\rX_d)$, we define its variance and minimum mean absolute error (MMAE) as 
\begin{align}
    \Var{\rX} &\defined \inf_{\ba \in \Reals^d} \EE{\|\rX - \ba\|_2^2}, \label{eq::defn_var_rv}\\
    \mmae(\rX) &\defined \inf_{\ba \in \Reals^d} \EE{\|\rX - \ba\|_1}. \label{eq::defn_mmae_rv}
\end{align}
The vector $\ba$ which minimizes \eqref{eq::defn_var_rv} and \eqref{eq::defn_mmae_rv} are
\begin{align}
    \argmin_{\ba \in \Reals^d} \EE{\|\rX - \ba\|_2^2} 
    &= (\EE{\rX_1},\cdots, \EE{\rX_d}), \\
    \argmin_{\ba \in \Reals^d} \EE{\|\rX - \ba\|_1}
    &= (\median(\rX_1),\cdots,\median(\rX_d)),
\end{align}
where $\median(\rX_i)$ is the median of the random variable $\rX_i$. 

In order to measure the difference between two probability distributions, we recall Csisz{\'a}r's $f$-divergence~\citep{csiszar1967information}. Let $f:(0,\infty)\to \Reals$ be a convex function with $f(1)=0$ and $P$, $Q$ be two probability distributions over a set $\mathcal{X}\subseteq \Reals^d$. The $f$-divergence between $P$ and $Q$ is defined as 
\begin{align}
    \textnormal{D}_{f}(P\|Q) \defined \int_{\mathcal{X}} f\left(\tfrac{\dif P}{\dif Q}\right) \dif Q.
\end{align}
Examples of $f$-divergence include KL-divergence ($f(t) = t\log t$), total variation distance ($f(t)=|t-1|/2$), and $\chi^2$-divergence ($f(t) = t^2-1$). 
The $f$-divergence motivates a way of measuring dependence between a pair of random variables $(\rX,\rY)$. Specifically, the $f$-information between $(\rX,\rY)$ is defined as 
\begin{align}
\label{eq::defn_f_inf}
    I_{f}(\rX;\rY) 
    \defined \textnormal{D}_{f}(P_{\rX,\rY}\|P_{\rX}\otimes P_{\rY})
    = \EE{\textnormal{D}_{f}(P_{\rY|\rX}\|P_{\rY})},
\end{align}
where $P_{\rX,\rY}$ is the joint distribution, $P_{\rX}$, $P_{\rY}$ are the marginal distributions, $P_{\rY|\rX}$ is the conditional distribution, and the expectation is taken over $\rX\sim P_{\rX}$. In particular, if the KL-divergence is used in \eqref{eq::defn_f_inf}, the corresponding $f$-information is the well-known mutual information \citep{shannon1948mathematical}.

\subsection{Information-theoretic Generalization Bounds} 

A recent work by \citet{xu2017information} provided a framework for analyzing algorithmic generalization. Specifically, they considered a learning algorithm as a channel (i.e., conditional probability distribution) that takes a training set $\rS$ as input and outputs a parameter $\rW$. Furthermore, they derived an upper bound for the expected generalization gap using the mutual information $I(\rW;\rS)$. This bound was later tightened by \citet{bu2020tightening} using an individual sample mutual information. Next, we recall the generalization bound in \citet{bu2020tightening}. By adapting their proof and leveraging variational representations of $f$-divergence \citep{nguyen2010estimating}, we present another two generalization bounds based on different kinds of $f$-information. Note that these two bounds can also be obtained from Corollary~1 in \citet{rodriguez2021tighter} and applying Jensen's inequality to Eq.~(15) in the same paper, respectively.
\begin{lemma}
\label{lem::Bu20}
Consider a learning algorithm which takes a data set $\rS=(\rZ_1,\cdots, \rZ_n)$ as input and outputs $\rW$.
\begin{itemize}
    \item \citep[Proposition~1 in][]{bu2020tightening} If the loss $\ell(\bw,\rZ)$ is $\sigma$-sub-Gaussian under $\rZ \sim \mu$ for all $\bw \in \mathcal{W}$, then
    \begin{align}
    \label{eq::Bu_gen_bound}
        \left|\EE{L_\mu(\rW)-L_{\rS}(\rW)}\right|
        \leq \frac{\sqrt{2}\sigma}{n}\sum_{i=1}^n \sqrt{I(\rW; \rZ_i)},
    \end{align}
    where $I(\rW; \rZ_i)$ is the mutual information (i.e., $f$-information with $f(t) = t\log t$).
    \item If the loss $\ell(\bw,\rZ)$ is upper bounded by a constant $A>0$, then
    \begin{align}
    \label{eq::T_inf_bound}
        \left|\EE{L_\mu(\rW)-L_{\rS}(\rW)}\right|
        \leq \frac{A}{n}\sum_{i=1}^n {\textnormal T}(\rW; \rZ_i),
    \end{align}
    where ${\textnormal T}(\rW; \rZ_i)$ is the ${\textnormal T}$-information (i.e., $f$-information with $f(t) = |t-1|/2$).
    \item If the variance of the loss function is finite (i.e., $\Var{\ell(\rW; \rZ)} < \infty$), then
    \begin{align}
    \label{eq::chi_inf_bound}
        \left|\EE{L_\mu(\rW)-L_{\rS}(\rW)}\right|
        \leq \frac{\sqrt{\Var{\ell(\rW; \rZ)}}}{n} \sum_{i=1}^n \sqrt{\chi^2(\rW;\rZ_i)},
    \end{align}
    where $\chi^2(\rW;\rZ_i)$ is the $\chi^2$-information (i.e., $f$-information with $f(t) = t^2-1$) and $\rZ$ is a fresh data point which is independent of $\rW$ (i.e., $(\rW,\rZ)\sim P_{\rW}\otimes \mu$).
\end{itemize}
\end{lemma}
\begin{proof}
See Appendix~\ref{append::proof_Bu20}.
\end{proof}
We apply Lemma~\ref{lem::Bu20} to analyze the generalization capability of noisy iterative algorithms. Estimating the $f$-information in Lemma~\ref{lem::Bu20} from data is intractable. Hence, we further upper bound these $f$-information by exploring properties of additive noise channels in the next section. Furthermore, we also incorporate a time-decaying factor into our bound, which is established by strong data processing inequalities, recalled in the upcoming subsection.

Although our analysis is applicable for \emph{any} $f$-information, we focus on the three $f$-information in Lemma~\ref{lem::Bu20} since:
\begin{itemize}[leftmargin=*]
\item Mutual information is often easier to work with due to its many useful properties. For example, the chain rule of mutual information plays an important role for handling sampling with replacement (see Section~\ref{subsec::SGLD}). 

\item ${\textnormal T}$-information often yields a tighter bound than \eqref{eq::Bu_gen_bound} and \eqref{eq::chi_inf_bound}. This can be seen by the following $f$-divergence inequalities \citep[see Eq.~1 and 94 in][]{sason2016f}:
\begin{align*}
    \sqrt{2} {\textnormal T}(\rW;\rZ_i) \leq \sqrt{I(\rW;\rZ_i)} \leq \sqrt{\log(1+\chi^2(\rW;\rZ_i))} \leq \sqrt{\chi^2(\rW;\rZ_i)}.
\end{align*}
Furthermore, the ${\textnormal T}$-information can be used to analyze a broader class of noisy iterative algorithms. For example, when the additive noise is drawn from a distribution with bounded support, the other two $f$-information may lead to an infinite generalization bound while the ${\textnormal T}$-information can still give a non-trivial bound (see the last row in Table~\ref{table:C_delta_exp}).

\item $\chi^2$-information requires the mildest assumptions. Apart from bounded loss functions, it is often hard to verify the sub-Gaussianity of $\ell(\bw,\rZ)$ for all $\bw$. The advantage of \eqref{eq::chi_inf_bound} is that it replaces the sub-Gaussian constant with the variance of the loss function.
\end{itemize}

\begin{remark}
Using $f$-information for bounding generalization gap has appeared in prior literature \citep[see e.g.,][]{alabdulmohsin2015algorithmic,jiao2017dependence,wang2019information,esposito2021generalization,rodriguez2021tighter,aminian2021jensen,jose2020information}. 
More broadly, there are significant recent works \citep[see e.g.,][]{raginsky2016information,asadi2018chaining,lopez2018generalization,steinke2020reasoning,hellstrom2020generalization,yagli2020information,hafez2020conditioning,jose2021information,zhou2021individually} on deriving new information-theoretic generalization bounds and applying them to different applications. 
The reason we adopt Lemma~\ref{lem::Bu20} for analyzing noisy iterative algorithms is that it enables us to incorporate a time-decaying factor into our bounds.
\end{remark}

\subsection{Strong Data Processing Inequalities} 

In order to characterize the time-decaying phenomenon, we use an information-theoretic tool: strong data processing inequalities \citep{dobrushin1956central,cohen1998comparisons}.
We start with recalling the data processing inequality.
\begin{lemma}
\label{lem::DPI}
If a Markov chain $\rU \to \rX \to \rY$ holds, then 
\begin{align}
    I_f(\rU; \rY) \leq I_f(\rU; \rX).
\end{align}
\end{lemma}
The data processing inequality states that no post-processing of $\rX$ can increase the information about $\rU$. Under certain conditions, the data processing inequality can be sharpened, which leads to a strong data processing inequality, often cast in terms of a contraction coefficient. Next, we recall the contraction coefficients of $f$-divergences and show their connection with strong data processing inequalities. 

For a given transition probability kernel $P_{\rY|\rX}:\mathcal{X}\to \mathcal{P}(\mathcal{Y})$ where $\mathcal{P}(\mathcal{Y})$ is the set of all distributions on $\mathcal{Y}$, let $P_{\rY|\rX}\circ P$ be the distribution on $\mathcal{Y}$ induced by the push-forward of the distribution $P$ (i.e., the distribution of $\rY$ when the distribution of $\rX$ is $P$). The contraction coefficient of $P_{\rY|\rX}$ for $\textnormal{D}_{f}$ is defined as
\begin{align*}
    \eta_f(P_{\rY|\rX}) \defined \sup_{P,Q:P\neq Q} \frac{\textnormal{D}_{f}(P_{\rY|\rX}\circ P\|P_{\rY|\rX}\circ Q)}{\textnormal{D}_{f}(P\|Q)} \in [0,1].
\end{align*}
In particular, when the total variation distance is used, the corresponding contraction coefficient $\eta_{\scalebox{.6}{\textnormal TV}}(P_{\rY|\rX})$ is known as the Dobrushin's coefficient \citep{dobrushin1956central}, which owns an equivalent expression:
\begin{align}
\label{eq::Dob_coe_eq}
    \eta_{\scalebox{.6}{\textnormal TV}}(P_{\rY|\rX}) 
    = \sup_{\bx, \bx'\in \mathcal{X}} \TV(P_{\rY|\rX=\bx} \| P_{\rY|\rX=\bx'}).
\end{align}
Note that the Dobrushin's coefficient upper bounds all other contraction coefficients \citep{cohen1998comparisons}:
\begin{align*}
    \eta_{f}(P_{\rY|\rX}) \leq \eta_{\scalebox{.6}{\textnormal TV}}(P_{\rY|\rX}).
\end{align*}
Furthermore, for any Markov chain $\rU\to \rX \to \rY$, the contraction coefficients satisfy \citep[see Theorem~5.2 in][for a proof]{raginsky2016strong}
\begin{align}
\label{eq::sdpi_mi_cont_coeff}
    I_f(\rU;\rY) 
    \leq \eta_f(P_{\rY|\rX}) \cdot I_f(\rU;\rX).
\end{align}
When $\eta_f(P_{\rY|\rX}) < 1$, the strict inequality $I_f(\rU;\rY) < I_f(\rU;\rX)$ improves the data processing inequality and, hence, is referred to as a strong data processing inequality. We refer the reader to \citet{polyanskiy2016dissipation} and \citet{raginsky2016strong} for a more comprehensive review on strong data processing inequalities and \citet{du2017strong} for non-linear strong data processing inequalities in Gaussian channels.

\section{Properties of Additive Noise Channels}
\label{sec::prop_add_channel}

Additive noise channels have a long history in information theory. Here we show two important properties of additive noise channels which will be used for deriving the generalization bounds in the next section. The first property (Lemma~\ref{lem::sdpi_diam}) leads to a decay factor into our bounds. The second property (Lemma~\ref{lem::gen_HWI}) produces computable generalization bounds.

Consider a single use of an additive noise channel. Let $(\rX,\rY)$ be a pair of random variables related by $\rY = \rX + m \rN$ where $\rX\in \mathcal{X}$; $m > 0$ is a constant; and $\rN$ represents an independent noise. In other words, the conditional distribution of $\rY$ given $\rX$ can be characterized by $P_{\rY|\rX=\bx} = P_{\bx + m\rN}$. 
If $\mathcal{X}$ is a compact set, the contraction coefficients often have a non-trivial upper bound, leading to a strong data processing inequality. This is formalized in the following lemma whose proof follows directly from the definition of the Dobrushin's coefficient in \eqref{eq::Dob_coe_eq} and the fact that the Dobrushin's coefficient is a universal upper bound of all the contraction coefficients. Note that the function $\delta(A,m)$ in \eqref{eq::defn_delta_At} upper bounds the Dobrushin's coefficient of the additive noise channel and has a closed-form expression for many noise distributions (see Table~\ref{table:C_delta_exp} for examples).
\begin{lemma}
\label{lem::sdpi_diam}
Let $\rN$ be a random variable which is independent of $(\rU,\rX)$. For a given norm $\|\cdot\|$ on a compact set $\mathcal{X}\subseteq \Reals^d$ and $m,A>0$, we define 
\begin{align}
\label{eq::defn_delta_At}
    \delta(A,m) &\defined \sup_{\|\bx - \bx'\|\leq A} \TV\left(P_{\bx + m\rN}\| P_{\bx' + m\rN}\right).
\end{align}
Then the Markov chain $\rU\to \rX \to \rX + m\rN$ holds and 
\begin{align}
    I_f(\rU;\rX + m\rN)\leq \delta(\mathsf{diam}(\mathcal{X}),m) \cdot I_f(\rU;\rX),
\end{align}
where $\mathsf{diam}(\mathcal{X}) \defined \sup_{\bx, \bx'\in \mathcal{X}} \|\bx - \bx'\|$ is the diameter of $\mathcal{X}$.
\end{lemma}
\begin{remark}
In the next section, we use Lemma~\ref{lem::sdpi_diam} to incorporate a time-decaying factor into our generalization bounds. The function $\delta(A,m)$ in \eqref{eq::defn_delta_At} yields a closed-form expression of the decay factor. As mentioned earlier, our analysis is inspired by a line of works on privacy amplification by iterations \citep{feldman2018privacy,asoodeh2020privacy}. For example, \citet{feldman2018privacy} proved that passing two probability distributions through a noisy iterative algorithm would shrink their R{\'e}nyi divergence, leading to amplification for R{\'e}nyi differential privacy. 
\citet{asoodeh2020privacy} reformulated the definition of differential privacy by using the $\mathsf{E}_\gamma$ divergence, which is a certain $f$-divergence. They characterized the contraction coefficient of $\mathsf{E}_\gamma$ by generalizing Dobrushin's coefficient (see Section~3 in their paper). 
In this paper, we introduce a general framework for deriving generalization bounds through $f$-divergences. We leverage Dobrushin's coefficient since it upper bounds all other contraction coefficients for $f$-divergences.
\end{remark}

Computing $f$-information in general is intractable when the underlying distribution is unknown. Hence, we further upper bound the $f$-information in Lemma~\ref{lem::Bu20} by a quantity which is easier to compute. To achieve this goal, we introduce another property of additive noise channels. Specifically, let $\rY = \rX + m\rN$ and $\rY' = \rX' + m\rN$ be the output variables from the same additive noise channel with input variables $\rX$ and $\rX'$, respectively. Then the $f$-divergence in the output space can be upper bounded by the optimal transport cost in the input space.
\begin{lemma}
\label{lem::gen_HWI}
Let $\rN$ be a random variable which is independent of $(\rX,\rX')$. For $\bx, \bx' \in \Reals^d$ and $m > 0$, we define a cost function\footnote{Note that $\mathsf{C}_f(\bx, \bx'; m)$ is not necessarily a metric.}
\begin{align}
\label{eq::cost_func}
    \mathsf{C}_f(\bx, \bx'; m) &\defined \textnormal{D}_{f}\left(P_{\bx + m\rN} \| P_{\bx' + m\rN}\right).
\end{align}
Then for any $m> 0$, we have
\begin{align}
    \textnormal{D}_{f}(P_{\rX+m\rN} \| P_{\rX'+m\rN})
    \leq \mathbb{W}(P_{\rX}, P_{\rX'}; m).
\end{align}
Here $\mathbb{W}(P_{\rX}, P_{\rX'}; m)$ is the optimal transport cost:
\begin{align}
\label{eq::defn_Wass}
    \mathbb{W}(P_{\rX}, P_{\rX'}; m) \defined \inf \EE{\mathsf{C}_f(\rX, \rX'; m)},
\end{align}
where the infimum is taken over all couplings (i.e., joint distributions) of the random variables $\rX$ and $\rX'$ with marginals $P_{\rX}$ and $P_{\rX'}$, respectively. 
\end{lemma}
\begin{proof}
See Appendix~\ref{append::gen_HWI}.
\end{proof}
\begin{table}
\small\centering
\resizebox{0.95\textwidth}{!}{
\renewcommand{\arraystretch}{1.5}
\begin{tabular}{lcccc}
\toprule
\texttt{Noise Type} & $\mathsf{C}_{\scalebox{.6}{\textnormal KL}}(\bx,\bx';m)$ & $\mathsf{C}_{\chi^2}(\bx,\bx';m)$ & $\mathsf{C}_{\scalebox{.6}{\textnormal TV}}(\bx,\bx';m)$ & $\delta(A,m)$\\
\toprule
Gaussian & $\frac{\|\bx - \bx'\|_2^2}{2m^2}$ & $\exp\left( \frac{\|\bx-\bx'\|_2^2}{m^2}\right) - 1$ & {\color{blue} $\frac{\|\bx - \bx'\|_2}{2m}$} & $1-2\bar{\Phi}\left(\frac{A}{2m}\right)$\\
\midrule
Laplace & {\color{blue} $\frac{\|\bx - \bx'\|_1}{m}$} & {\color{blue} $\exp\left(\frac{\|\bx - \bx'\|_1}{m}\right) - 1$} & {\color{blue} $\sqrt{\frac{\|\bx - \bx'\|_1}{2m}}$} & {\color{blue} $1 - \exp\left(-\frac{A}{m}\right)$}\\
\midrule
Uniform on $[-1,1]$ & $\infty\indicator{[x\neq x']}$ & $\infty\indicator{[x\neq x']}$ & $\min\left\{1,\left|\frac{x-x'}{2m}\right|\right\}$ & $\min\left\{1,\frac{A}{2m}\right\}$ \\
\bottomrule 
\end{tabular}
}
\caption{\small{Closed-form expressions (or upper bounds if in blue color) of $\mathsf{C}_f(\bx, \bx'; m)$ (see \eqref{eq::cost_func} for its definition) and $\delta(A,m)$ (see \eqref{eq::defn_delta_At} for its definition). The function $\delta(A,m)$ is equipped with the 2-norm for Gaussian distribution and 1-norm for Laplace distribution. We denote the Gaussian complementary cumulative distribution function (CCDF) by $\bar{\Phi}(x)\defined \int^{\infty}_{x} \frac{1}{\sqrt{2\pi}} \exp(-v^2/2)\dif v$ and define $\infty \cdot 0 = 0$ as convention. The proof is deferred to Appendix~\ref{append::table_C_delta}.}
}
\label{table:C_delta_exp}
\end{table}

Lemmas~\ref{lem::sdpi_diam} and \ref{lem::gen_HWI} show that the functions $\delta(A,m)$ and $\mathsf{C}_f(\bx, \bx'; m)$ can be useful for sharpening the data processing inequality and upper bounding the $f$-information in Lemma~\ref{lem::Bu20}. We demonstrate in Table~\ref{table:C_delta_exp} that these functions can be expressed in closed-form for specific additive noise distributions.

\begin{remark}
Let $\rN$ be drawn from a Gaussian distribution. Substituting the closed-form expression of $\mathsf{C}_{\scalebox{.6}{\textnormal KL}}(\bx,\bx';m)$ from Table~\ref{table:C_delta_exp} into Lemma~\ref{lem::gen_HWI} leads to 
\begin{align}
\label{eq::HWI_conv_Gaussian}
    \KL(P_{\rX + m\rN} \| P_{\rX' + m\rN})\leq \frac{1}{2m^2} \mathbb{W}_2^2(P_{\rX}, P_{\rX'})
\end{align}
where $\mathbb{W}_2(P_{\rX},P_{\rX'})$ is the 2-Wasserstein distance equipped with the $L_2$ cost function: 
\begin{align*}
    \mathbb{W}_2^2(P_{\rX},P_{\rX'}) \defined \inf \EE{\|\rX - \rX'\|_2^2}.
\end{align*}
This inequality serves as a fundamental building block for proving Otto-Villani's HWI inequality \citep{otto2000generalization} in the Gaussian case \citep{raginsky2012concentration,boucheron2013concentration}.
\end{remark}

\section{Generalization Bounds for Noisy Iterative Algorithms}
\label{sec::Gen_Bound}

In this section, we present our main result---generalization bounds for noisy iterative algorithms. 
First, by leveraging strong data processing inequalities, we prove that the amount of information about the data points used in early iterations decays with time. Accordingly, our generalization bounds incorporate a time-decaying factor which enables the impact of early iterations on our bounds to reduce with time. Second, by using properties of additive noise channels developed in the last section, we further upper bound the $f$-information by a quantity which is often easier to estimate. The above two aspects correspond to Lemma~\ref{lem::Noisy_ite_alg_SDPI} and \ref{lem::ub_MI_WT_Zi_general} which are the basis of our main result in Theorem~\ref{thm::gen_bound_general_noisy_alg}. 

Before diving into the analysis, we first discuss assumptions made in this paper.
\begin{assumption}[Sampling w/o Replacement]
\label{assump::index_w/o_rep}
The mini-batch indices $(\mathcal{B}_1, \cdots, \mathcal{B}_T)$ in \eqref{eq::rec_update_parameter_Extend} are specified before the algorithm is run and data are drawn without replacement. 
\end{assumption}
If the mini-batches are selected when the algorithm is run, one can analyze the expected generalization gap by first conditioning on $\mathcal{B} \defined (\mathcal{B}_1, \cdots, \mathcal{B}_T)$ and then taking an expectation over the randomness of $\mathcal{B}$:
\begin{align*}
    \EE{L_\mu(\rW_T)-L_{\rS}(\rW_T)}
    = \EE{\EE{L_\mu(\rW_T)-L_{\rS}(\rW_T)\mid \mathcal{B}}}.
\end{align*}
Our analysis can be extended to the case where data are drawn with replacement (see Proposition~\ref{prop::gen_bound_SGLD}) by using the chain rule for mutual information. 
\begin{assumption}[Bounded Gradient \& Compact Domain]
\label{assump::comp_W_Lip_grad}
The parameter domain $\mathcal{W}$ is compact and $\|g(\bw, \bz)\|\leq K$ for all $\bw,\bz$. We denote the diameter of $\mathcal{W}$ by $D \defined \sup_{\bw, \bw' \in \mathcal{W}} \|\bw - \bw'\|$. 
\end{assumption}
Our generalization bounds rely on the second assumption mildly. In fact, this assumption only affects the time-decaying factor in our bounds which is always upper bounded by $1$. If we remove this assumption, our bounds still hold though the decay factor disappears.

Now we are in a position to derive generalization bounds under the above assumptions. As a consequence of strong data processing inequalities, the following lemma indicates that the information of a data point $\rZ_i$ contained in the algorithmic output $\rW_T$ will reduce with time $T$.
\begin{lemma}
\label{lem::Noisy_ite_alg_SDPI}
Under Assumption~\ref{assump::index_w/o_rep} (Sampling w/o Replacement) and \ref{assump::comp_W_Lip_grad} (Bounded Gradient \& Compact Domain), if a data point $\rZ_i$ is used in the $t$-th iteration, then
\begin{align}
\label{eq::MI_time_decay_prod}
    I_f(\rW_T; \rZ_i)
    \leq I_f(\rW_t; \rZ_i) \cdot \prod_{t'=t+1}^T \delta(D+2\eta_{t'}K, m_{t'}),
\end{align}
where the function $\delta(\cdot,\cdot)$ is defined in \eqref{eq::defn_delta_At}.
\end{lemma}
\begin{proof}
For the $t$-th iteration, we rewrite the recursion in \eqref{eq::rec_update_parameter_Extend} as
\begin{subequations}
\label{eq::para_update_UVW}
\begin{align}
    &\rU_t = \rW_{t-1} - \eta_t\cdot g(\rW_{t-1}, \{\rZ_{i}\}_{i \in \mathcal{B}_t}) \\
    &\rV_t = \rU_t + m_t \cdot \rN_t \\
    &\rW_t = \mathsf{Proj}_{\mathcal{W}} \left(\rV_t \right).
\end{align}
\end{subequations}
Let $\rZ_i$ be a data point used at the $t$-th iteration. Under Assumption~\ref{assump::index_w/o_rep} (Sampling w/o Replacement), the following Markov chain holds:
\begin{align}
\label{eq::Markov_chain}
    \rZ_i \to \rU_{t} \to \rV_{t}\to \rW_{t} \to \cdots \to \rW_{T-1} \to \rU_{T} \to \rV_{T} \to \rW_T.
\end{align}
Let $\mathcal{U}_T$ be the range of $\rU_T$. By Assumption~\ref{assump::comp_W_Lip_grad} (Bounded Gradient \& Compact Domain) and the triangle inequality,
\begin{align*}
    \mathsf{diam}(\mathcal{U}_T)
    \leq \mathsf{diam}(\mathcal{W}) + 2 \eta_T K 
    = D + 2 \eta_T K.
\end{align*}
Now we leverage the strong data processing inequality in Lemma~\ref{lem::sdpi_diam} and obtain 
\begin{align*}
    I_f(\rW_T; \rZ_i) 
    &\leq I_f(\rV_T; \rZ_i)\\
    &\leq \delta(D+2\eta_T K, m_T) \cdot I_f(\rU_T; \rZ_i)\\
    &\leq \delta(D+2\eta_T K, m_T) \cdot I_f(\rW_{T-1}; \rZ_i),
\end{align*}
where the first and last steps are due to the data processing inequality (Lemma~\ref{lem::DPI}). Applying this procedure recursively leads to the desired conclusion.
\end{proof}

For many types of noise (e.g., Gaussian or Laplace noise), the function $\delta(\cdot,\cdot)$ is \emph{strictly} smaller than $1$ (see Table~\ref{table:C_delta_exp}). In this case, the information about the data points used in early iterations is reducing via the multiplicative factor in \eqref{eq::MI_time_decay_prod}. Furthermore, one can even prove that $I_f(\rW_T;\rZ_i)\to 0$ as $T\to \infty$ if the magnitude of the additive noise in \eqref{eq::rec_update_parameter_Extend} has a lower bound.

Lemma~\ref{lem::Noisy_ite_alg_SDPI} explains how our generalization bounds in Theorem~\ref{thm::gen_bound_general_noisy_alg} incorporate a time-decaying factor. However, it still involves an $f$-information $I_f(\rW_t;\rZ_i)$, which can be hard to compute from data. Next, we further upper bound this $f$-information by using properties of additive noise channels developed in the last section (see Lemma~\ref{lem::gen_HWI}).
\begin{lemma}
\label{lem::ub_MI_WT_Zi_general}
Under Assumption~\ref{assump::index_w/o_rep} (Sampling w/o Replacement), if a data point $\rZ_i$ is used at the $t$-th iteration, then 
\begin{align}
    I_f(\rW_t; \rZ_i) 
    \leq \EE{\mathsf{C}_f\left(g(\rW_{t-1}, \rZ), g(\rW_{t-1},\bar{\rZ}); \frac{m_tb_t}{\eta_t}\right)},
\end{align}
where the function $\mathsf{C}_f(\cdot,\cdot;\cdot)$ is defined in \eqref{eq::cost_func} and the expectation is taken over $(\rW_{t-1},\rZ,\bar{\rZ})\sim P_{\rW_{t-1}}\otimes \mu \otimes \mu$.
\end{lemma}
\begin{proof}
Recall the definition of $\rU_t$, $\rV_t$ in \eqref{eq::para_update_UVW}. The data processing inequality yields
\begin{align}
    I_f(\rW_t; \rZ_i) 
    &\leq I_f(\rV_{t}; \rZ_i). \label{eq::MI_sdpi_WTZi_VT}
\end{align}
By the definition of $f$-information, we can write 
\begin{align}
\label{eq::ml_KL_Vt_Z}
    I_f(\rV_{t}; \rZ_i) 
    = \EE{\textnormal{D}_{f}(P_{\rV_{t}|\rZ_i} \| P_{\rV_{t}})}
    = \int_{\mathcal{Z}} \textnormal{D}_{f}(P_{\rV_{t}|\rZ_i=\bz} \| P_{\rV_{t}}) \dif \mu(\bz).
\end{align}
Since $\rV_t = \rU_t + m_t \cdot \rN_t$ by its definition, Lemma~\ref{lem::gen_HWI} leads to
\begin{align}
\label{eq::KL_ub_wass_gen}
    \textnormal{D}_{f}\left(P_{\rV_{t}|\rZ_i = \bz} \| P_{\rV_{t}} \right)
    \leq \mathbb{W}(P_{\rU_t| \rZ_i = \bz}, P_{\rU_t}; m_t).
\end{align}
To further upper bound the above optimal transport cost, we construct a special coupling. Let $\rW_{t-1}$ be the output of the noisy iterative algorithm at the $(t-1)$-st iteration. Under Assumption~\ref{assump::index_w/o_rep} (Sampling w/o Replacement), the data point $\rZ_i$ is independent of $\rW_{t-1}$ since it is only used at the $t$-th iteration. Then we introduce two random variables:
\begin{align*}
    &\rU^*_{\bz} 
    \defined \rW_{t-1} - \frac{\eta_{t}}{b_{t}} \Bigg(\sum_{\substack{j\in \mathcal{B}_t,j\neq i}} g(\rW_{t-1}, \rZ_j) + g(\rW_{t-1}, \bz)\Bigg),\\
    &\rU^* 
    \defined \rW_{t-1} - \frac{\eta_{t}}{b_{t}} \sum_{j\in \mathcal{B}_t} g(\rW_{t-1}, \rZ_j).
\end{align*}
Here $\rU^*_{\bz}$ and $\rU^*$ have marginals $P_{\rU_t|\rZ_i = \bz}$ and $P_{\rU_t}$, respectively. By the definition of optimal transport cost in \eqref{eq::defn_Wass}, we have
\begin{align}
\label{eq::Wass_ub_Exp_Cm}
    &\mathbb{W}\left(P_{\rU_{t}|\rZ_i = \bz}, P_{\rU_{t}}; m_t\right)
    \leq \EE{\mathsf{C}_f(\rU_{\bz}^*, \rU^*; m_t)}.
\end{align}
The property of $\mathsf{C}_f(\bx, \by; m)$ in Lemma~\ref{lem::prop_Ct} yields
\begin{align}
    \EE{\mathsf{C}_f(\rU_{\bz}^*, \rU^*; m_t)} 
    &= \EE{\mathsf{C}_f\left(-\frac{\eta_t}{b_t} g(\rW_{t-1}, \bz), -\frac{\eta_t}{b_t} g(\rW_{t-1}, \rZ_i); m_t\right)} \nonumber\\
    &= \EE{\mathsf{C}_f\left(\frac{\eta_t}{b_t} g(\rW_{t-1}, \rZ_i), \frac{\eta_t}{b_t} g(\rW_{t-1},\bz); m_t\right)} \nonumber\\
    &= \EE{\mathsf{C}_f\left(g(\rW_{t-1}, \rZ_i), g(\rW_{t-1}, \bz); \frac{m_tb_t}{\eta_t}\right)}. \label{eq::Exp_Cmt_Cmtbtetat}
\end{align}
We introduce two independent copies $\rZ, \bar{\rZ}$ of $\rZ_i$ such that $(\rW_{t-1}, \rZ, \bar{\rZ}) \sim P_{\rW_{t-1}} \otimes \mu \otimes \mu$. Combining (\ref{eq::ml_KL_Vt_Z}--\ref{eq::Exp_Cmt_Cmtbtetat}) and using Tonelli's theorem lead to
\begin{align}
    I_f(\rV_{t}; \rZ_i)
    &\leq \EE{\mathsf{C}_f\left(g(\rW_{t-1}, \rZ), g(\rW_{t-1},\bar{\rZ}); \frac{m_tb_t}{\eta_t}\right)}. \label{eq::I_Vt_Zi_ub_C_g}
\end{align}
Substituting \eqref{eq::I_Vt_Zi_ub_C_g} into \eqref{eq::MI_sdpi_WTZi_VT} gives the desired conclusion.
\end{proof}

With Lemma~\ref{lem::Bu20}, \ref{lem::Noisy_ite_alg_SDPI}, and \ref{lem::ub_MI_WT_Zi_general} in hand, we now present the main result in this section: three generalization bounds for noisy iterative algorithms under different assumptions.
\begin{theorem}
\label{thm::gen_bound_general_noisy_alg}
Suppose that Assumption~\ref{assump::index_w/o_rep} (Sampling w/o Replacement) and \ref{assump::comp_W_Lip_grad} (Bounded Gradient \& Compact Domain) hold.
\begin{itemize}
\item If the loss $\ell(\bw,\rZ)$ is $\sigma$-sub-Gaussian under $\rZ \sim \mu$ for all $\bw \in \mathcal{W}$, the expected generalization gap $\EE{L_\mu(\rW_T)-L_{\rS}(\rW_T)}$ can be upper bounded by
\begin{align}
\label{eq::KL_bound_ite_alg}
    \frac{\sqrt{2}\sigma}{n} \sum_{t=1}^T b_t \sqrt{ \EE{\mathsf{C}_{\scalebox{.6}{\textnormal KL}}\left(g(\rW_{t-1}, \rZ), g(\rW_{t-1},\bar{\rZ}); \frac{m_tb_t}{\eta_t}\right)} \prod_{t'=t+1}^T \delta(D+2\eta_{t'}K, m_{t'})}.
\end{align}

\item If the loss function is upper bounded by a constant $A>0$, the expected generalization gap $\EE{L_\mu(\rW_T)-L_{\rS}(\rW_T)}$ can be upper bounded by
\begin{align}
\label{eq::TV_bound_ite_alg}
    \frac{A}{n} \sum_{t=1}^T b_t \EE{\mathsf{C}_{\scalebox{.6}{\textnormal TV}}\left(g(\rW_{t-1}, \rZ), g(\rW_{t-1},\bar{\rZ}); \frac{m_tb_t}{\eta_t}\right)} \prod_{t'=t+1}^T \delta(D+2\eta_{t'}K, m_{t'}).
\end{align}

\item If the variance of the loss function is finite (i.e., $\Var{\ell(\rW_T;\rZ)}<\infty$), the expected generalization gap $\EE{L_\mu(\rW_T)-L_{\rS}(\rW_T)}$ can be upper bounded by
\begin{align}
\label{eq::chi_bound_ite_alg}
    \frac{\sigma}{n} \sum_{t=1}^T b_t \sqrt{\EE{\mathsf{C}_{\chi^2}\left(g(\rW_{t-1}, \rZ), g(\rW_{t-1},\bar{\rZ}); \frac{m_tb_t}{\eta_t}\right)} \prod_{t'=t+1}^T \delta(D+2\eta_{t'}K, m_{t'})},
\end{align}
where $\sigma \defined \sqrt{\Var{\ell(\rW_T;\rZ)}}$ with $(\rW_{T},\rZ)\sim P_{\rW_{T}}\otimes \mu$.
\end{itemize}
\end{theorem}
\begin{proof}
See Appendix~\ref{append::gen_bound_general_noisy_alg}.
\end{proof}
Our generalization bounds involve the information (e.g., step size $\eta_t$, magnitude of noise $m_t$, and batch size $b_t$) at all iterations. Moreover, our bounds depend on the data distribution $\mu$ through the expectation terms. Under Assumption~\ref{assump::index_w/o_rep} (Sampling w/o Replacement), if a data point $\rZ_i$ is used in the $t$-th iteration, it is independent of $\rW_{t-1}$ (i.e., $P_{\rW_{t-1},\rZ_i} = P_{\rW_{t-1}} \otimes \mu$). This is why the expectations are taken over $(\rW_{t-1},\rZ,\bar{\rZ})\sim P_{\rW_{t-1}}\otimes \mu \otimes \mu$. Finally, since the function $\delta$ is often strictly smaller than $1$ (see Table~\ref{table:C_delta_exp} for some examples), the multiplicative factor $\prod_{t'=t+1}^T \delta(D+2\eta_{t'}K, m_{t'})$ enables the impact of early iterations on our bounds to reduce with time $T$.

The generalization bounds in Theorem~\ref{thm::gen_bound_general_noisy_alg} may seem contrived at first glance as they rely on the functions $\delta$ and $\mathsf{C}_f$ defined in \eqref{eq::defn_delta_At} and \eqref{eq::cost_func}. However, in the next section, we will show that these bounds can be significantly simplified when we apply them to real applications. Furthermore, we will also compare the advantage of each bound under these applications.

\section{Applications}
\label{sec::application}

We demonstrate the generalization bounds in Theorem~\ref{thm::gen_bound_general_noisy_alg} through several applications in this section. 

\subsection{Differentially Private Stochastic Gradient Descent (DP-SGD)}
\label{subsec::DP_SGD}

Differentially private stochastic gradient descent (DP-SGD) is a variant of SGD where noise is added to a stochastic gradient estimator in order to ensure privacy of each individual record. We recall an implementation of (projected) DP-SGD \citep[see e.g., Algorithm~1 in][]{feldman2018privacy}. At each iteration, the parameter of the empirical risk is updated using the following rule:
\begin{align}
\label{eq::DP_SGD_alg}
    \rW_t = \mathsf{Proj}_{\mathcal{W}} \left(\rW_{t-1} - \eta \left(g(\rW_{t-1}, \{\rZ_{i}\}_{i \in \mathcal{B}_t}) + \rN_t \right) \right),
\end{align}
where $\rN_t$ is an additive noise drawn independently from a distribution $P_{\rN}$; $\mathcal{B}_t$ contains the indices of the data points used at the current iteration and $b_t \defined |\mathcal{B}_t|$; the function $g$ indicates a direction for updating the parameter. The recursion in \eqref{eq::DP_SGD_alg} is run for $T$ iterations and we assume that data are drawn without replacement. At the end of each iteration, the parameter is projected onto a compact domain $\mathcal{W}$. We denote the diameter of $\mathcal{W}$ by $D$. The output from the DP-SGD algorithm is the last iterate $\rW_T$. Finally, we assume that 
\begin{align}
\label{eq::DP_Lip}
    \sup_{\bw\in\mathcal{W},\bz\in\mathcal{Z}}\|g(\bw,\bz)\| \leq K.
\end{align}
This assumption can be satisfied by gradient clipping and is crucial for guaranteeing differential privacy as it controls the sensitivity of each update. 

The differential privacy guarantees of the DP-SGD algorithm have been extensively studied in the literature \citep[see e.g.,][]{song2013stochastic,wu2017bolt,feldman2018privacy,balle2019privacy,asoodeh2020privacy}. Here we consider a different angle: the generalization of DP-SGD. We derive generalization bounds for the DP-SGD algorithm under Laplace and Gaussian mechanisms by using our Theorem~\ref{thm::gen_bound_general_noisy_alg}.
\begin{proposition}[Laplace mechanism]
\label{prop::DP_SGD_Lap}
Suppose that the additive noise $\rN_t$ in \eqref{eq::DP_SGD_alg} follows a standard multivariate Laplace distribution. Let $\mathcal{W}$ be equipped with the 1-norm and $q \defined 1 - \exp\left(-(D+2\eta K)/\eta\right) \in (0,1)$.
\begin{itemize}
    \item If the loss $\ell(\bw, \rZ)$ is $\sigma$-sub-Gaussian under $\rZ\sim \mu$ for all $\bw \in \mathcal{W}$, then
    \begin{align}
        \EE{L_\mu(\rW_T)-L_{\rS}(\rW_T)}
        \leq \frac{2\sigma}{n} \sum_{t=1}^T \sqrt{b_t\cdot  \mmae\left(g(\rW_{t-1}, \rZ)\right) \cdot q^{T-t}}.
    \end{align}
    
    \item If the loss function is upper bounded by $A>0$, then
    \begin{align}
        \EE{L_\mu(\rW_T)-L_{\rS}(\rW_T)}
        \leq \frac{\sqrt{2}A}{n} \sum_{t=1}^T \sqrt{b_t}\cdot \EE{\sqrt{\left\|g(\rW_{t-1}, \rZ) -\be\right\|_1}} \cdot q^{T-t},
    \end{align}
    where $\be \defined \median\left(g(\rW_{t-1}, \rZ)\right)$.

    \item If the variance of the loss function is bounded (i.e., $\Var{\ell(\rW_T;\rZ)} < \infty$), then
    \begin{align}
    \label{eq::gen_bound_chi_Lap}
        \EE{L_\mu(\rW_T)-L_{\rS}(\rW_T)}
        \leq \frac{\sigma}{n} \sum_{t=1}^T \sqrt{b_t \cdot \EE{\exp\left(2\left\|g(\rW_{t-1}, \rZ) - \be \right\|_1\right) - 1}  \cdot q^{T-t}},
    \end{align}
    where $\sigma = \sqrt{\Var{\ell(\rW_T;\rZ)}}$ and $\be \defined \median\left(g(\rW_{t-1}, \rZ)\right)$.
\end{itemize}
\end{proposition}
\begin{proof}
See Appendix~\ref{append::gen_bound_DP_SGD}.
\end{proof}

\begin{proposition}[Gaussian mechanism]
\label{prop::DP_SGD_Gaussian}
Suppose that the additive noise $\rN_t$ in \eqref{eq::DP_SGD_alg} follows a standard multivariate Gaussian distribution. Let $\mathcal{W}$ be equipped with the 2-norm and $q \defined 1 - 2\bar{\Phi}\left((D+2\eta K)/2\eta\right) \in (0,1)$ with $\bar{\Phi}(\cdot)$ being the Gaussian CCDF.
\begin{itemize}
    \item If the loss $\ell(\bw, \rZ)$ is $\sigma$-sub-Gaussian under $\rZ\sim \mu$ for all $\bw \in \mathcal{W}$, then
    \begin{align}
    \label{eq::gen_bound_KL_Gauss}
        \EE{L_\mu(\rW_T)-L_{\rS}(\rW_T)}
        \leq \frac{2\sigma}{n} \sum_{t=1}^T \sqrt{\Var{g(\rW_{t-1}, \rZ)} \cdot q^{T-t}}.
    \end{align}
    
    \item If the loss function is upper bounded by $A>0$, then
    \begin{align}
    \label{eq::gen_bound_TV_Gauss}
        \EE{L_\mu(\rW_T)-L_{\rS}(\rW_T)}
        \leq \frac{A}{n} \sum_{t=1}^T \EE{\left\|g(\rW_{t-1}, \rZ) - \be\right\|_2} \cdot q^{T-t},
    \end{align}
    where $\be \defined \EE{g(\rW_{t-1}, \rZ)}$.
    
    \item If the variance of the loss function is bounded (i.e., $\Var{\ell(\rW_T;\rZ)} < \infty$), then
    \begin{align}
    \label{eq::gen_bound_chi_Gauss}
        \EE{L_\mu(\rW_T)-L_{\rS}(\rW_T)}
        \leq \frac{\sigma}{n} \sum_{t=1}^T \sqrt{\EE{\exp\left(4\left\|g(\rW_{t-1}, \rZ) - \be \right\|_2^2\right) - 1}  \cdot q^{T-t}},
    \end{align}
    where $\sigma = \sqrt{\Var{\ell(\rW_T;\rZ)}}$ and $\be \defined \EE{g(\rW_{t-1}, \rZ)}$.
\end{itemize}
\end{proposition}
\begin{proof}
See Appendix~\ref{append::gen_bound_DP_SGD}.
\end{proof}
Our Theorem~\ref{thm::gen_bound_general_noisy_alg} leads to three generalization bounds for each DP-SGD mechanism. We discuss the advantage of each bound in the following remark by focusing on the Gaussian mechanism.
\begin{remark}
\label{rem::DPSGD_Comp}
We first assume that the loss function is upper bounded by $A$, leading to an $A/2$-sub-Gaussian loss $\ell(\bw,\rZ)$ and $\sqrt{\Var{\ell(\rW_T;\rZ)}} \leq A/2$. Since
\begin{align*}
    \EE{\rX} \leq \sqrt{\EE{\rX^2}} \leq \frac{1}{2} \sqrt{\EE{\exp(4\rX^2) - 1}},
\end{align*}
then for $\be \defined \EE{g(\rW_{t-1}, \rZ)}$
\begin{align*}
    \EE{\left\|g(\rW_{t-1}, \rZ) - \be\right\|_2} 
    \leq \sqrt{\Var{g(\rW_{t-1}, \rZ)}}
    \leq \frac{1}{2} \sqrt{\EE{\exp\left(4\left\|g(\rW_{t-1}, \rZ) - \be \right\|_2^2\right) - 1}}.
\end{align*}
Therefore, we have
\begin{align*}
    \eqref{eq::gen_bound_TV_Gauss} \leq \eqref{eq::gen_bound_KL_Gauss} \leq \eqref{eq::gen_bound_chi_Gauss}.
\end{align*}
In other words, the total-variation bound in \eqref{eq::TV_bound_ite_alg} yields the tightest generalization bound \eqref{eq::gen_bound_TV_Gauss} for the DP-SGD algorithm. On the other hand, the $\chi^2$-divergence bound in \eqref{eq::chi_bound_ite_alg} leads to a bound \eqref{eq::gen_bound_chi_Gauss} that requires the mildest assumption. At this moment, it seems unclear what the advantage of the KL-divergence bound is. Nonetheless, we will show in Section~\ref{subsec::SGLD} that the nice properties of mutual information (e.g., chain rule) help extend our analysis to the general setting where data are drawn with replacement. 
\end{remark}

A standard approach \citep[see e.g.,][]{he2021tighter} for analyzing the generalization of the DP-SGD algorithm often follows two steps: establish $(\epsilon,\delta)$-differential privacy guarantees for the DP-SGD algorithm and prove/apply a generalization bound that holds for \emph{any} $(\epsilon,\delta)$-differentially private algorithms. 
However, generalization bounds obtained in this manner are distribution-independent since differential privacy is robust with respect to the data distribution. As observed in existing literature \citep[see e.g.,][]{zhang2016understanding} and our Figure~\ref{Fig::label_corruption}, machine learning models trained under different data distributions can exhibit completely different generalization behaviors. Our bounds take into account the data distribution through the expectation terms (or mmae, variance).

Our generalization bounds can be estimated from data. Take the bound in \eqref{eq::gen_bound_KL_Gauss} as an example. If sufficient data are available at each iteration, we can estimate the variance term by the population variance of $\{g(\rW_{t-1}, \rZ_{i}) \mid i \in \mathcal{B}_t\}$ since $\rW_{t-1}$ is independent of $\rZ_i$ for $i \in \mathcal{B}_t$. Alternatively, we can draw a hold-out set for estimating the variance term at each iteration.

\subsection{Federated Learning (FL)}
\label{subsec::FL}

Federated learning (FL) \citep{mcmahan2017communication} is a setting where a model is trained across multiple clients (e.g., mobile devices) under the management of a central server while the training data are kept decentralized. We recall the federated averaging algorithm with local-update DP-SGD in Algorithm~\ref{alg:FL_SGLD} and refer the readers to \citet{kairouz2019advances} for a more comprehensive review.
\begin{algorithm}[httb]
\begingroup
\small
\caption{Federated averaging (local DP-SGD).}
\label{alg:FL_SGLD}

\begin{algorithmic}[*]

\State {\bfseries Input:} 

\State \quad Total number of clients $N$ and clients per round $C$

\State \quad Total global updates $T$ and local updates $M$

\State \quad DP-SGD learning rate $\eta$

\vspace{0.25em}

\State \textbf{Initialize:} $\rW_0$ randomly selected from $\mathcal{W}$

\For{$t = 1, \cdots, T$ global steps}

\State Server chooses a subset $\mathcal{S}_t$ of $C$ clients

\State Server sends $\rW_{t-1}$ to all selected clients

\For{each client $k \in \mathcal{S}_t$ in parallel}

\State Initialize $\rW_{t,0}^{k} \gets \rW_{t-1}$

\For{$j = 1, \cdots, M$ local steps}

\State Draw $b$ fresh data points $\{\rZ_i^k\}_{i\in [b]}$ and noise $\rN_{t,j}^{k} \sim N(0, \mathbf{I}_d)$

\State Update the parameter $\rW_{t,j}^{k} \gets \mathsf{Proj}_{\mathcal{W}} \left(\rW_{t,j-1}^{k} - \eta \left(g\left(\rW_{t,j-1}^{k}, \{\rZ_i^k\}_{i\in [b]}\right) + \rN_{t,j}^{k}\right) \right)$

\EndFor

\State Send $\rW_{t,M}^{k}$ back to the server

\EndFor

\State Server aggregates the parameter $\rW_{t} = \frac{1}{C} \sum_{k\in \mathcal{S}_t} \rW_{t,M}^{k}$

\EndFor

\State {\bfseries Output:} $\rW_T$

\end{algorithmic}

\endgroup

\end{algorithm}

It is crucial to be able to \emph{monitor} the performance of the global model on each client. Although the global model could achieve a desirable performance on average, it may fail to achieve high accuracy for each local client. 
This is because in the federated learning setting, data are typically unbalanced (different clients own different number of samples) and not identically distributed (data distribution varies across different clients). 
Since in practice clients may not have an extra hold-out data set to evaluate the performance of the global model, they can instead compute the loss of the model on their training set and compensate the mismatch by the generalization gap (or its upper bound). 
It is worth noting that this approach of monitoring model performance is completely decentralized as the clients do not need to share their data with the server and all the computation can be done locally.
As discussed in Remark~\ref{rem::DPSGD_Comp}, the total variation bound in \eqref{eq::TV_bound_ite_alg} often leads to the tightest generalization bound so we recast it under the setting of FL. 
\begin{proposition}
\label{prop::gen_bound_FL}
Let $\mathcal{T}_k \subset [T]$ contain the indices of global iterations in which the $k$-th client interacts with the server. If the loss function is upper bounded by $A>0$, the expected generalization gap of the $k$-th client has an upper bound:
\begin{align*}
    \EE{L_{\mu_k}(\rW_T)-L_{\rS_k}(\rW_T)}
    \leq
    \frac{A}{n_k} \sum_{t \in \mathcal{T}_k} \sum_{j=1}^M \EE{\|g(\rW_{t,j-1}^{k}, \rZ^k) - \be\|_2} \cdot q^{M(T+1-t)-j},
\end{align*}
where $n_k$ is the number of training data from the $k$-th client, $\be \defined \EE{g(\rW_{t,j-1}^{k}, \rZ^k)}$, and 
\begin{align*}
    q \defined 1-2\bar{\Phi}\left(\frac{\sqrt{C}(D+2\eta K)}{2\eta} \right) \in (0,1)
\end{align*}
with $D$ being the diameter of $\mathcal{W}$, $K \defined \sup_{\bw,\bz}\|g(\bw, \bz)\|_2$, and $\bar{\Phi}(\cdot)$ being the Gaussian CCDF.
\end{proposition}
\begin{proof}
See Appendix~\ref{append::gen_bound_FL}.
\end{proof}
\citet{yagli2020information} introduced a generalization bound in the context of federated learning. However, their bound in Theorem~3 involves a mutual information. Here we replace the mutual information with an expectation term. This improvement allows local clients to compute our bound from their training data reliably.

\subsection{Stochastic Gradient Langevin Dynamics (SGLD)}
\label{subsec::SGLD}

We analyze the generalization gap of the stochastic gradient Langevin dynamics (SGLD) algorithm \citep{gelfand1991recursive,welling2011bayesian}. We start by recalling a standard framework of SGLD. The data set $\rS$ is first divided into $m$ disjoint mini-batches:
\begin{align*}
    \rS=\bigcup_{j=1}^m \rS_j,\quad \text{where } |\rS_j|=b \text{ and } \rS_j \cap \rS_{k} = \emptyset \text{ for } j\neq k.
\end{align*}
We initialize the parameter of the empirical risk with a random point $\rW_0 \in \mathcal{W}$ and update using the following rule:
\begin{align}
\label{eq::SGLD_para_up_rule}
    \rW_t = \rW_{t-1} - \eta_t \nabla_{\bw} \hat{\ell}(\rW_{t-1}, \rS_{B_t}) + \sqrt{\frac{2\eta_t}{\beta_t}} \rN_t,
\end{align}
where $\eta_t$ is the learning rate; $\beta_t$ is the inverse temperature; $\rN_t$ is drawn independently from a standard Gaussian distribution; $B_t \in [m]$ is the mini-batch index; $\hat{\ell}$ is a surrogate loss (e.g., hinge loss); and 
\begin{align}
\label{eq::grad_update_minibatch}
    \nabla_{\bw} \hat{\ell}(\rW_{t-1}, \rS_{B_t}) 
    \defined \frac{1}{b} \sum_{\rZ \in \rS_{B_t}} \nabla_{\bw} \hat{\ell}(\rW_{t-1}, \rZ).
\end{align}

We study a general setting where the output from SGLD can be any function of the parameters across all iterations (i.e., $\rW = f(\rW_1,\cdots,\rW_T)$), including the setting considered before where $\rW = \rW_T$. For example, the output can be an average of all iterates (i.e., Polyak averaging) $\rW = \frac{1}{T}\sum_{t} \rW_t$ or the parameter which achieves the smallest value of the loss function $\rW = \argmin_{\rW_t} L_{\mu}(\rW_t)$.

Alas, Theorem~\ref{thm::gen_bound_general_noisy_alg} cannot be applied directly to the SGLD algorithm because the Markov chain in \eqref{eq::Markov_chain} does not hold any more when data are drawn with replacement. In order to circumvent this issue, we develop a different proof technique by using the chain rule for mutual information. 
\begin{proposition}
\label{prop::gen_bound_SGLD}
If the loss function $\ell(\bw, \rZ)$ is $\sigma$-sub-Gaussian under $\rZ\sim \mu$ for all $\bw \in \mathcal{W}$, then
\begin{align*}
    \EE{L_\mu(\rW)-L_{\rS}(\rW)}
    \leq \frac{\sqrt{2b}\sigma}{2n} \sum_{j=1}^m \sqrt{\sum_{t\in \mathcal{T}_j} \beta_t \eta_t \cdot \Var{\nabla_{\bw} \hat{\ell}(\rW_{t-1}, \rS_j)}},
\end{align*}
where the set $\mathcal{T}_j$ contains the indices of iterations in which the mini-batch $\rS_j$ is used.
\end{proposition}
\begin{proof}
See Appendix~\ref{append::gen_bound_SGLD}.
\end{proof}

Our bound incorporates the gradient variance which measures a particular kind of ``flatness'' of the loss landscape. We note that a recent work \citep{jiang2019fantastic} has observed empirically that the variance of gradients is predictive of and highly correlated with the generalization gap of neural networks. Here we evidence this connection from a theoretical viewpoint by incorporating the gradient variance into the generalization bound.

Unfortunately, our generalization bound does not incorporate a decay factor anymore.\footnote{We note that the analysis in \citet{mou2018generalization} requires $\rW=\rW_T$. Hence, in the setting we consider (i.e., $\rW$ is a function of $\rW_1,\cdots,\rW_T$), it is unclear if it is possible to include a decay factor in the bound.} To understand why it happens, let us imagine an extreme scenario in which the SGLD algorithm outputs all the iterates (i.e., $\rW = (\rW_1,\cdots,\rW_T)$). For a data point $\rZ_i$ used at the $t$-th iteration, the data processing inequality implies that
\begin{align*}
    I(\rW_1,\cdots,\rW_T;\rZ_i) \geq I(\rW_t;\rZ_i).
\end{align*}
Hence, it is impossible to have $I(\rW_1,\cdots,\rW_T;\rZ_i)\to 0$ as $T\to \infty$ unless $I(\rW_t;\rZ_i) = 0$.

Many existing SGLD generalization bounds \citep[e.g.,][]{mou2018generalization,li2019generalization,pensia2018generalization,negrea2019information} are expressed as a sum of errors associated with each training iteration. In order to compare with these results, we present an analogous bound in the following corollary. This bound is obtained by combining a key lemma for proving Proposition~\ref{prop::gen_bound_SGLD} with Minkowski inequality and Jensen's inequality so it is often much weaker than Proposition~\ref{prop::gen_bound_SGLD}.
\begin{corollary}
\label{cor::SGLD_gen_bound_trajectory}
If the loss function $\ell(\bw, \rZ)$ is $\sigma$-sub-Gaussian under $\rZ\sim \mu$ for all $\bw \in \mathcal{W}$, the expected generalization gap of the SGLD algorithm can be upper bounded by 
\begin{align*}
    \frac{\sqrt{2}\sigma}{2} \min\left\{\frac{1}{n} \sum_{t=1}^T \sqrt{\beta_t \eta_t \cdot \Var{\nabla_{\bw} \hat{\ell}(\rW_{t-1}, \rZ^{\dagger}_{t})}},
    \sqrt{\frac{1}{b n}\sum_{t=1}^T\beta_t \eta_t \cdot \Var{\nabla_{\bw} \hat{\ell}(\rW_{t-1}, \rZ^{\dagger}_{t})}}\right\},
\end{align*}
where $\rZ^{\dagger}_t$ is any data point used in the $t$-th iteration.
\end{corollary}
\begin{proof}
See Appendix~\ref{append::SGLD_gen_bound_trajectory}.
\end{proof}

Our bound is distribution-dependent through the variance of gradients in contrast with Corollary~1 of \citet{pensia2018generalization}, Proposition~3 of \citet{bu2020tightening}, and Theorem~1 of \citet{mou2018generalization}, which rely on the Lipschitz constant: $\sup_{\bw,\bz} \|\nabla_{\bw} \hat{\ell}(\bw,\bz)\|_2$. These bounds fail to explain some generalization phenomena of DNNs, such as label corruption \citep{zhang2016understanding}, because the Lipschitz constant takes a supremum over all possible weight matrices $\bw$ and data points $\bz$. In other words, this Lipschitz constant only relies on the architecture of the network instead of the weight matrices or data distribution. Hence, it is the same for a network trained from corrupted data and a network trained from true data. We remark that the Lipschitz constant used by \citet{pensia2018generalization,bu2020tightening,mou2018generalization} is different from the Lipschitz constant of the function corresponding to a network w.r.t. the input variable. The latter one has been used in the literature \citep[see e.g.,][]{bartlett2017spectrally} for deriving generalization bounds and, to some degree, can capture generalization phenomena, such as label corruption.

The order of our generalization bound in Corollary~\ref{cor::SGLD_gen_bound_trajectory} is $\min\Big(\frac{1}{n} \sum_{t=1}^T \sqrt{\beta \eta_t}, \sqrt{\frac{\beta}{bn} \sum_{t=1}^T \eta_t} \Big)$. It is tighter than Theorem~2 of \citet{mou2018generalization} whose order is $\sqrt{\frac{\beta}{n} \sum_{t=1}^T \eta_t}$. Our bound is applicable regardless of the choice of learning rate while the bound in \citet{li2019generalization} requires the scale of the learning rate to be upper bounded by the reciprocal of the Lipschitz constant. Our Corollary~\ref{cor::SGLD_gen_bound_trajectory} has the same order with \citet{negrea2019information} but we incorporate an additional decay factor when applying our bounds to the DP-SGD algorithm (see Proposition~\ref{prop::DP_SGD_Gaussian}) and numerical experiments suggest that our bound is more favourably correlated with the true generalization gap (see Table~\ref{table:comparison}).

\section{Numerical Experiments}
\label{sec::experiments}

In this section, we demonstrate our generalization bound (Proposition~\ref{prop::gen_bound_SGLD}) through numerical experiments on the MNIST data set \citep{lecun1998gradient}, CIFAR-10 data set \citep{krizhevsky2009learning}, and SVHN data set \citep{netzer2011reading}, showing that it can predict the behavior of the true generalization gap.

\subsection{Corrupted Labels}

\begin{figure*}[h]
\centering
\includegraphics[width=0.32\linewidth]{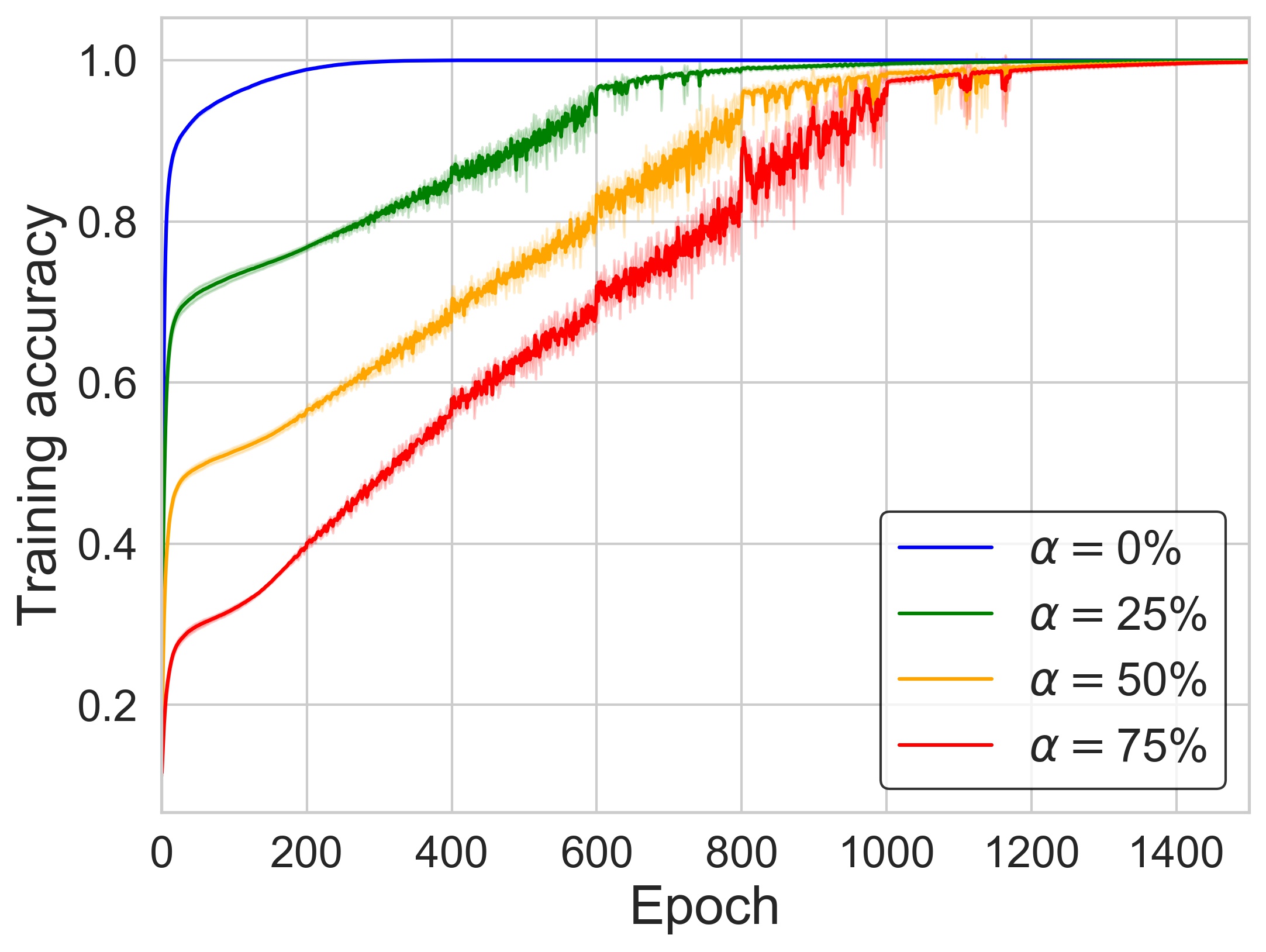}
\includegraphics[width=0.32\linewidth]{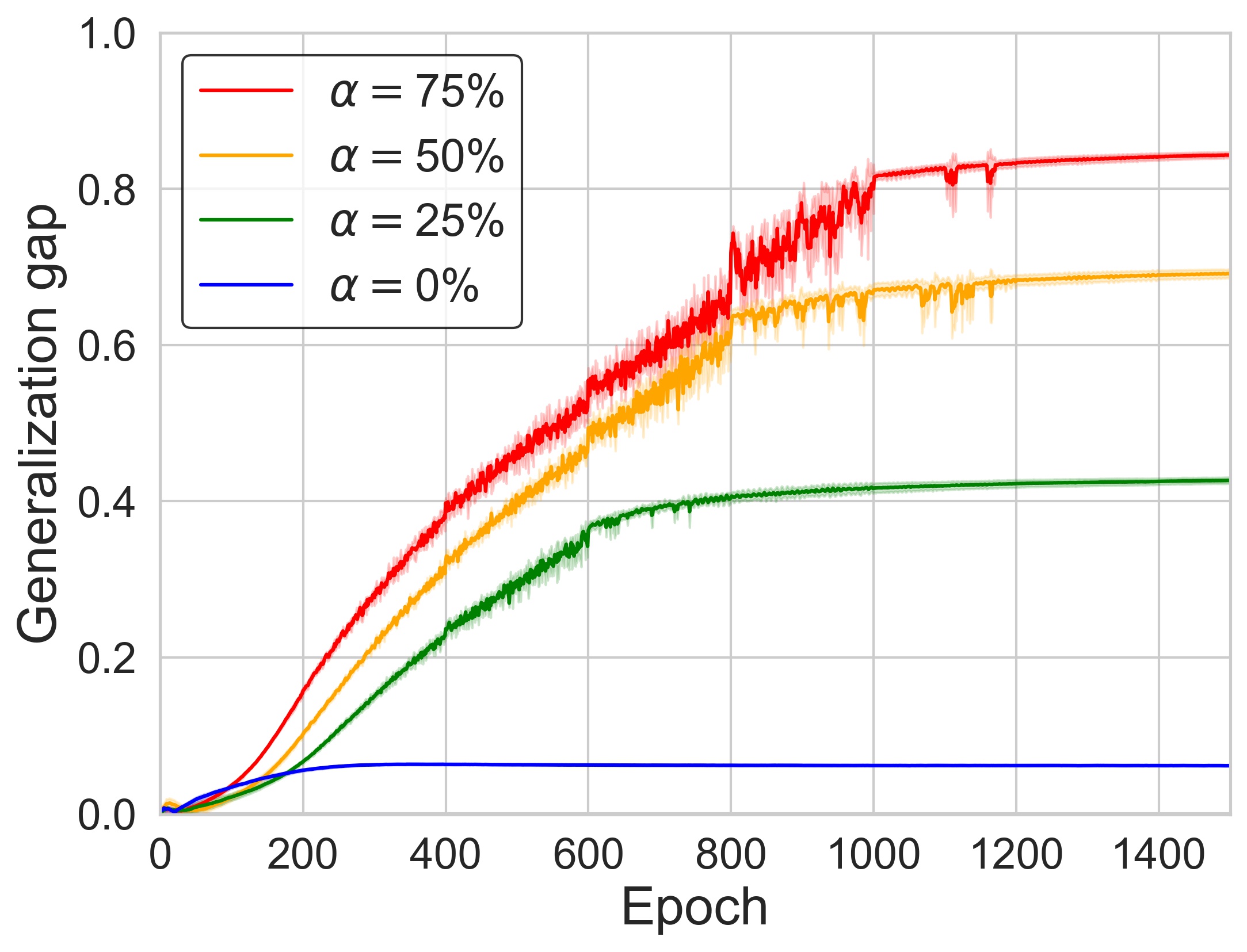}
\includegraphics[width=0.32\linewidth]{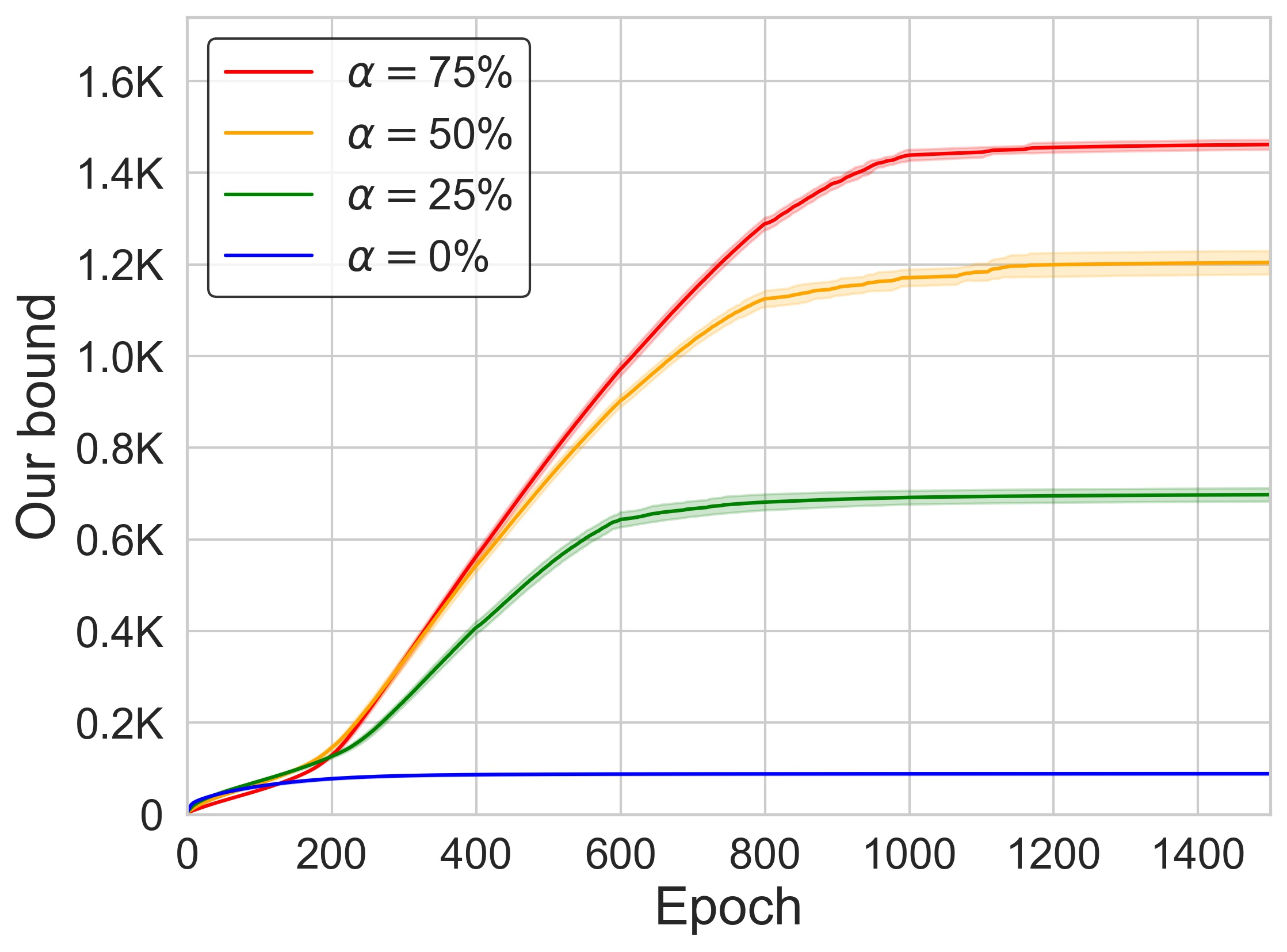}
\includegraphics[width=0.32\linewidth]{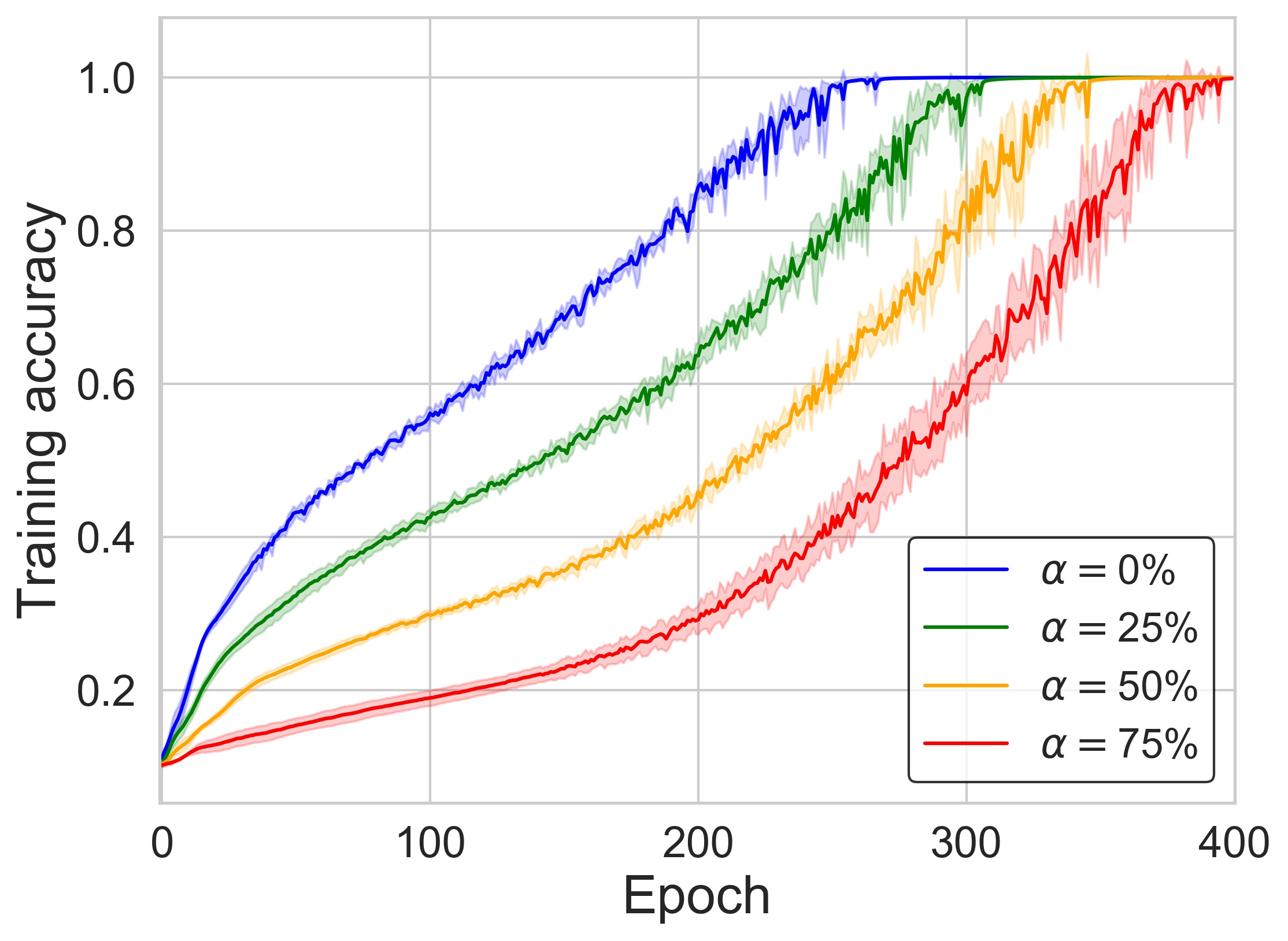}
\includegraphics[width=0.32\linewidth]{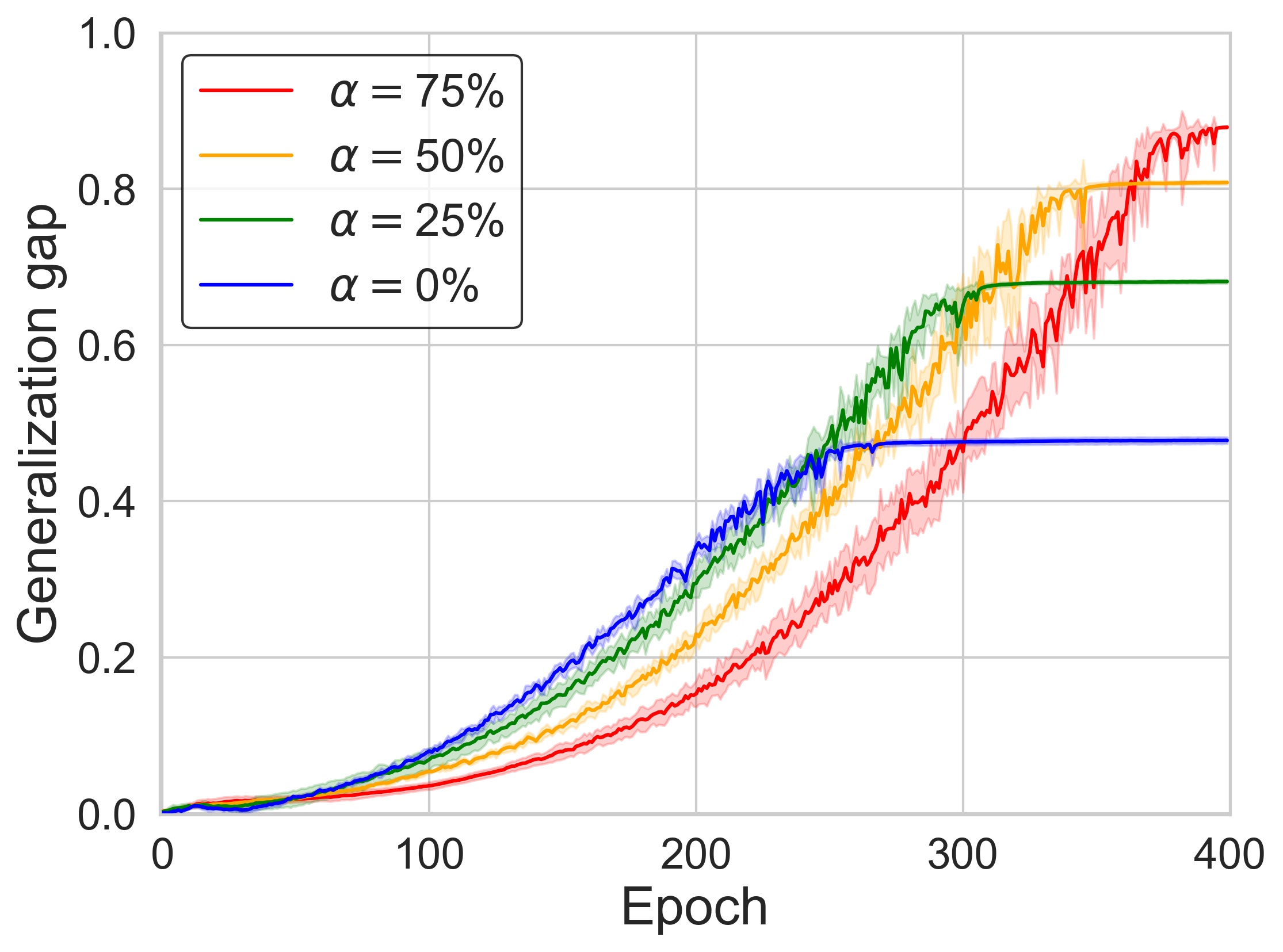}
\includegraphics[width=0.32\linewidth]{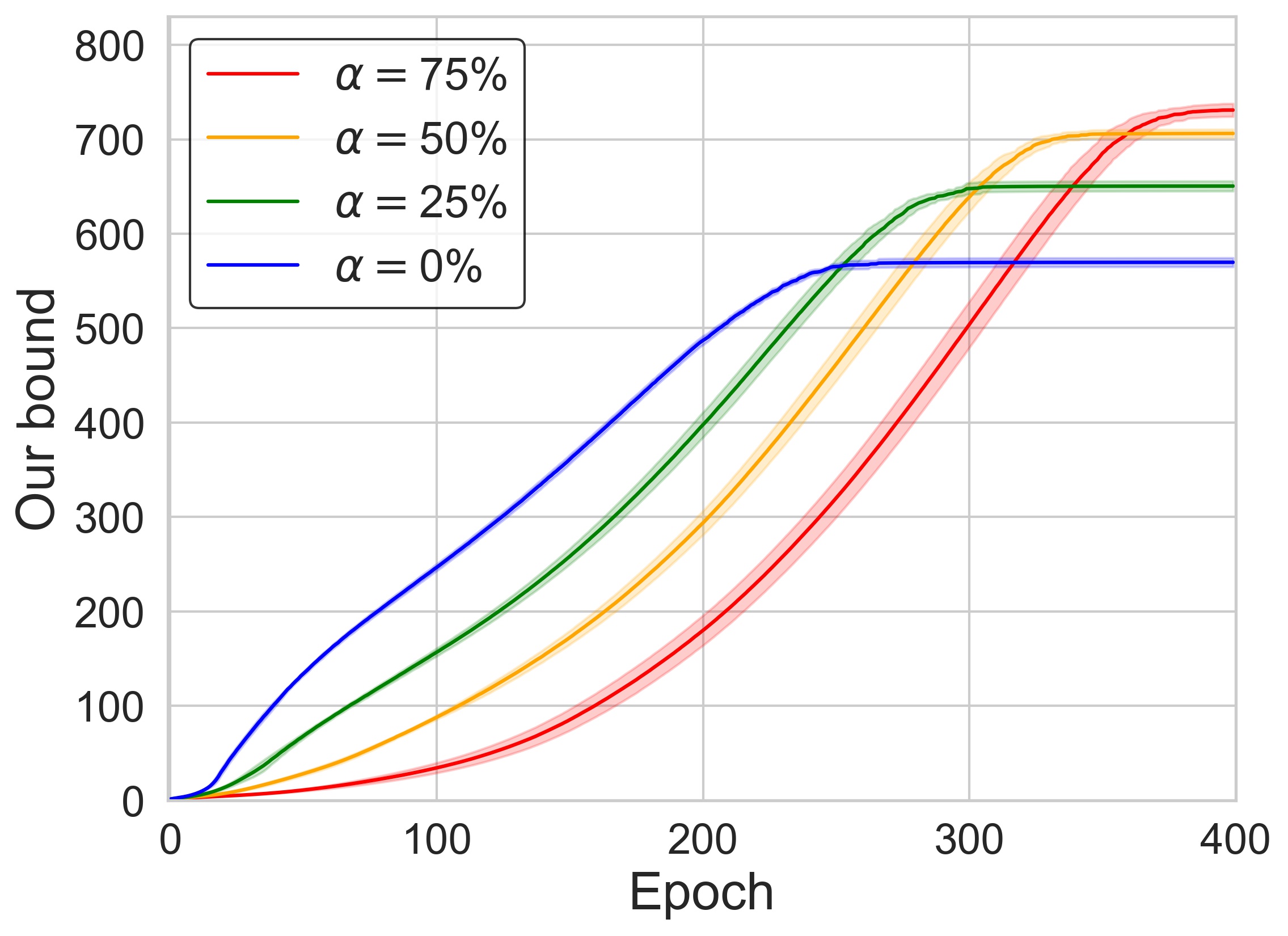}
\includegraphics[width=0.32\linewidth]{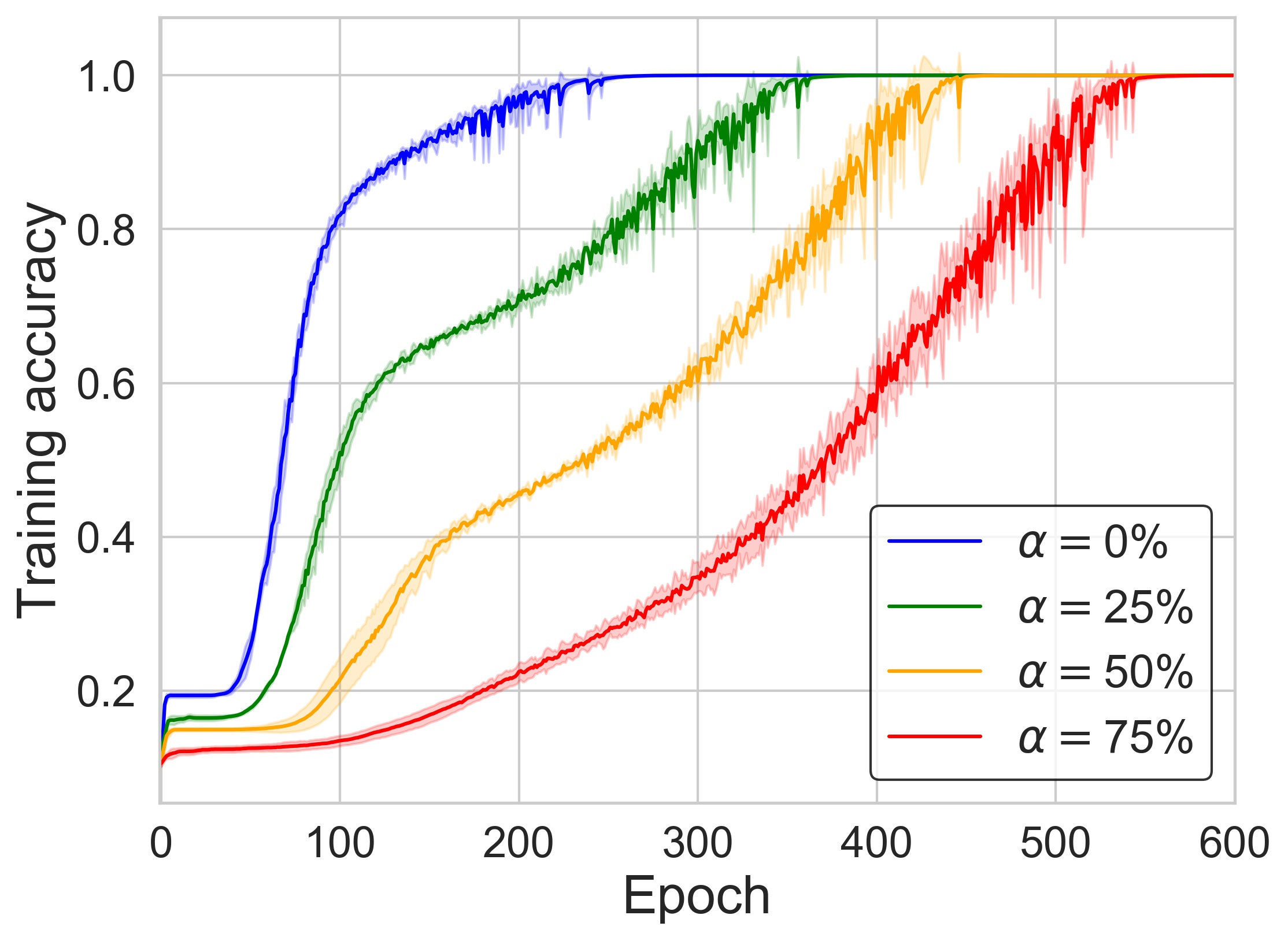}
\includegraphics[width=0.32\linewidth]{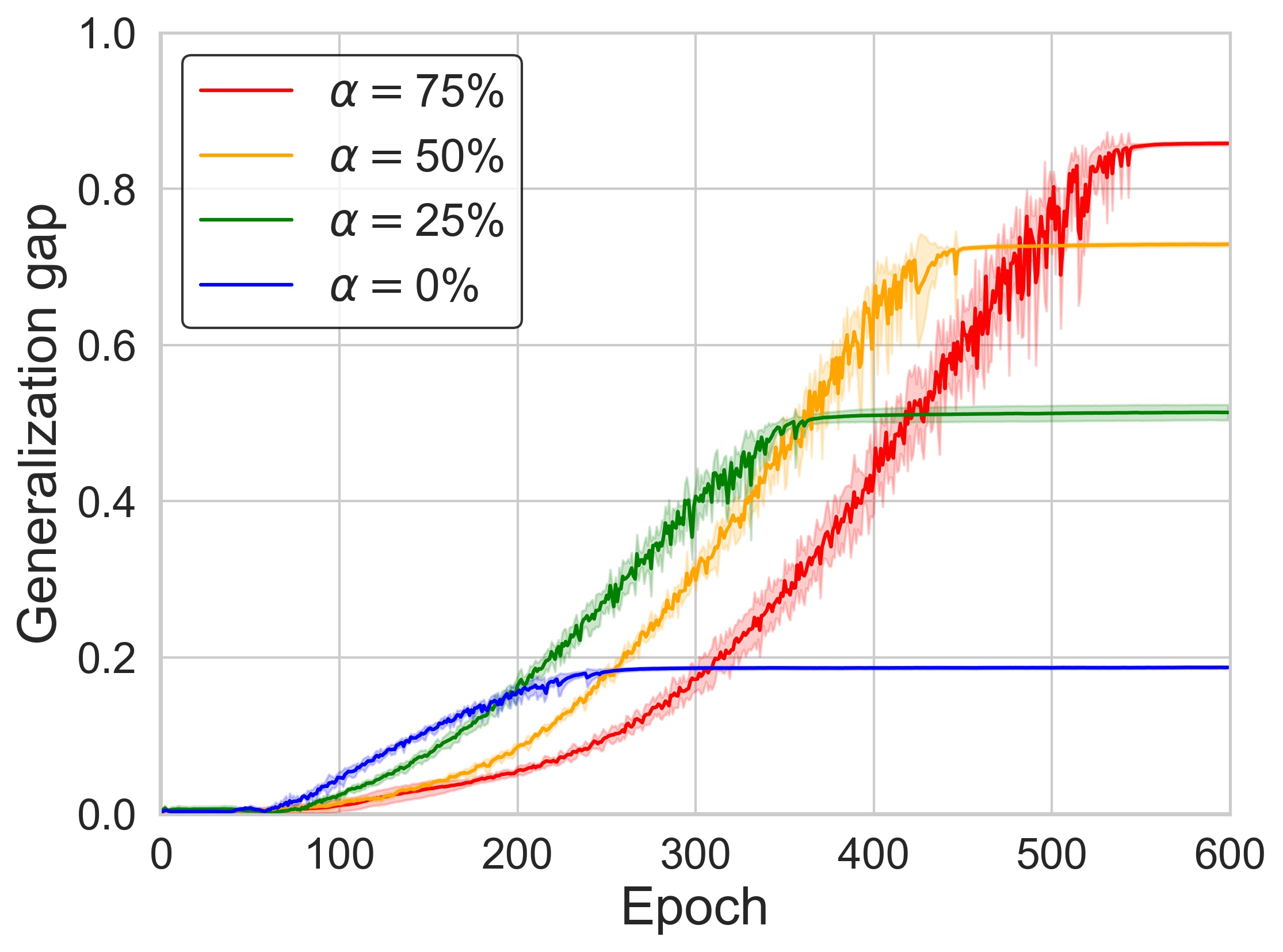}
\includegraphics[width=0.32\linewidth]{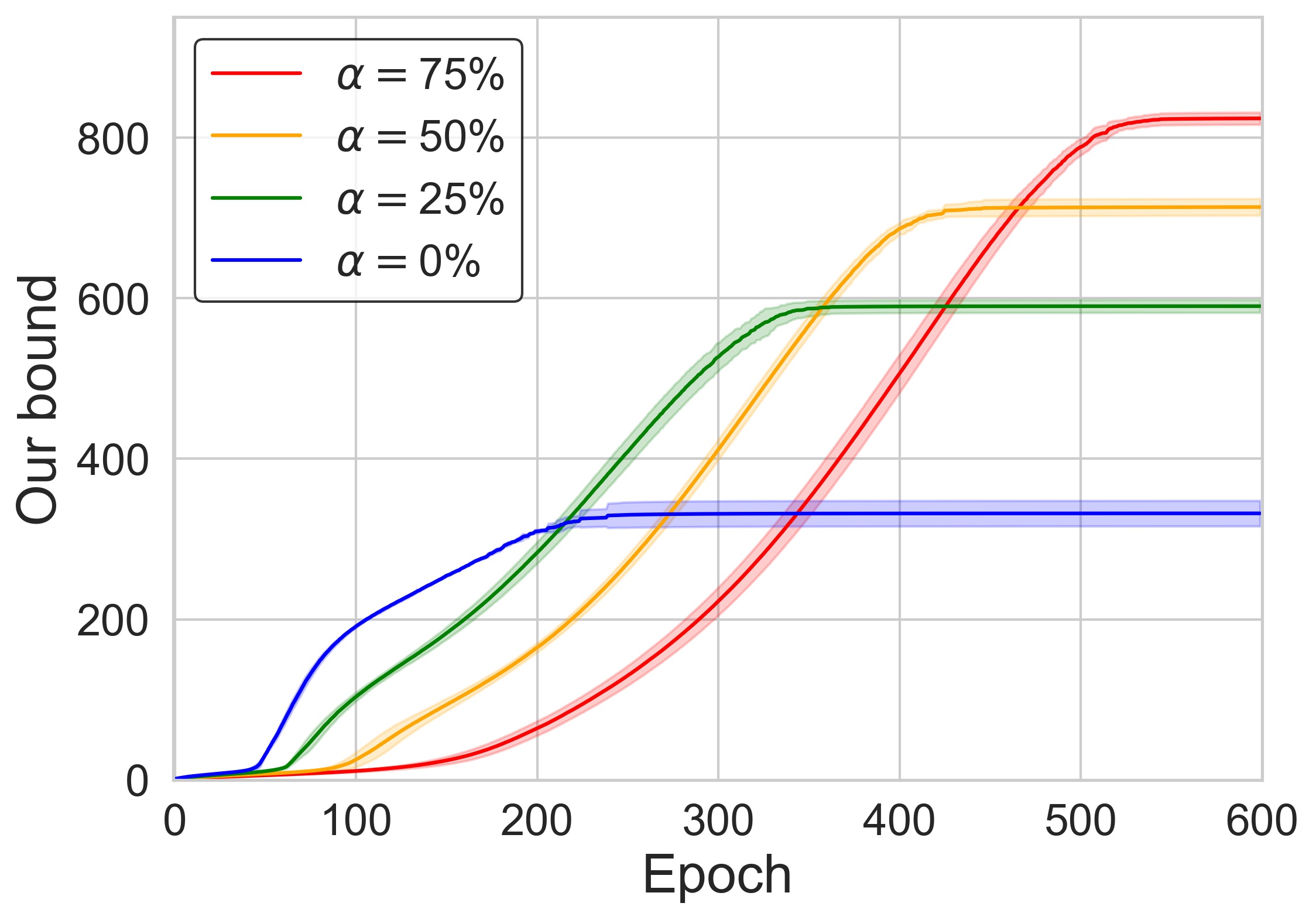}
\caption{\small{Illustration of our generalization bound in Proposition~\ref{prop::gen_bound_SGLD}. We use the SGLD algorithm to train 3-layer neural networks on MNIST (top row) and convolutional neural networks on CIFAR-10 (middle row) and SVHN (bottom row) when the training data have different label corruption level $\alpha\in\{0\%,25\%,50\%,75\%\}$. Left column: training accuracy. Middle column: empirical generalization gap. Right column: empirical generalization bound.}
}
\label{Fig::label_corruption}
\end{figure*}

As observed in \citet{zhang2016understanding}, DNNs have the potential to memorize the entire training data set even when a large portion of the labels are corrupted. For networks with identical architecture, those trained using true labels have better generalization capability than those ones trained using corrupted labels, although both of them achieve perfect training accuracy. Unfortunately, distribution-independent bounds, such as the ones using VC-dimension, may not be able to capture this phenomenon because they are invariant for both true data and corrupted data. In contrast, our bound quantifies this empirical observation, exhibiting a lower value on networks trained on true labels compared to ones trained on corrupted labels (Figure~\ref{Fig::label_corruption}).

In our experiment, we randomly select 5000 samples as our training data set and change the label of $\alpha\in\{0\%,25\%,50\%,75\%\}$ of the training samples. Then we use the SGLD algorithm to train a neural network under different corruption level. 
The training process continues until the training accuracy is $1.0$ (see Figure~\ref{Fig::label_corruption} Left). 
We compare our generalization bound with the generalization gap in Figure~\ref{Fig::label_corruption} Middle and Right. 
As shown, both our bound and the generalization gap are increasing w.r.t. the corruption level in the last epoch. 
Furthermore, the curve of our bound has very similar shape with the generalization gap. 
Finally, we observe that the generalization gap tends to be stable since the algorithm converges (Figure~\ref{Fig::label_corruption} Middle). Our generalization bound captures this phenomenon  (Figure~\ref{Fig::label_corruption} Right) as the variance of gradients becomes negligible when the algorithm starts converging. The intuition is that the variance of gradients reflects the flatness of the loss landscape and as the algorithm converges, the loss landscape becomes flatter.

\begin{figure}[h]
\centering
\includegraphics[width=0.32\linewidth]{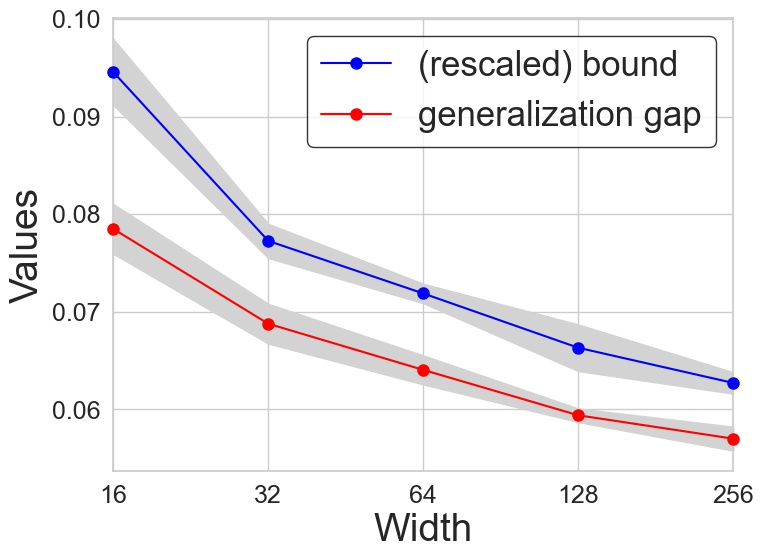}
\includegraphics[width=0.32\linewidth]{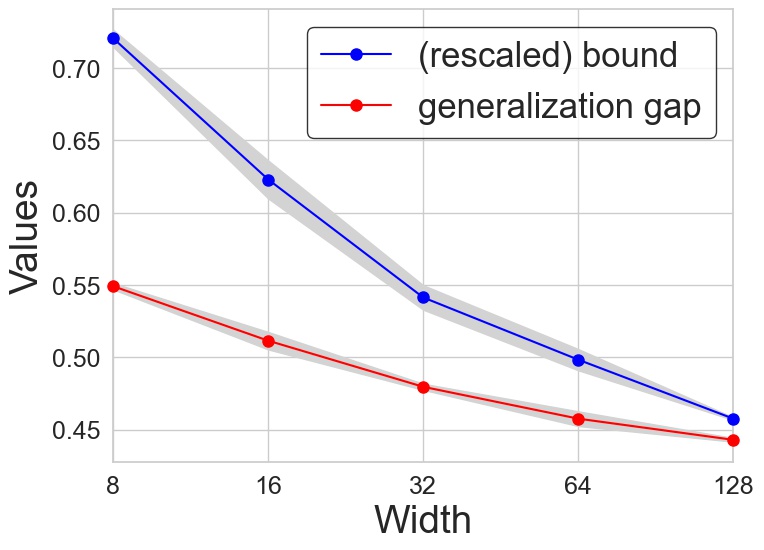}
\includegraphics[width=0.32\linewidth]{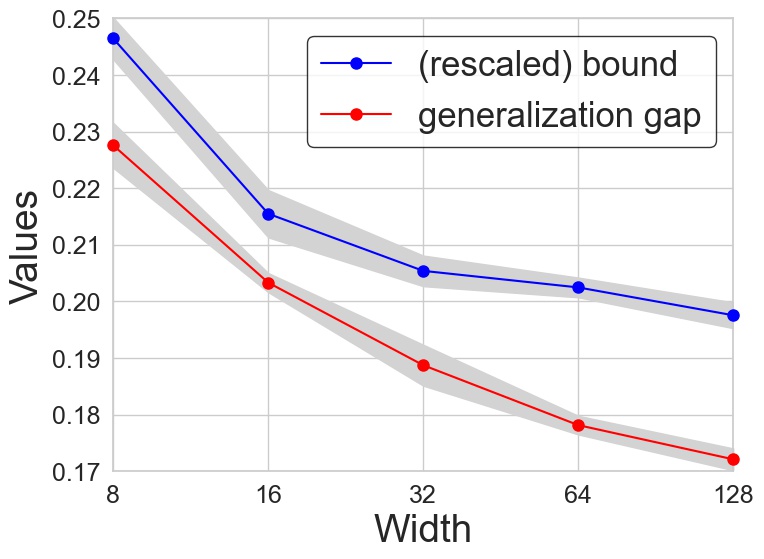}
\caption{\small{The generalization gap of neural networks with varying widths trained on MNIST (left), CIFAR-10 (middle), and SVHN (right) using the SGLD algorithm. To compare with, we plot the generalization bound (Proposition~\ref{prop::gen_bound_SGLD}) rescaled by multiplying $7\text{e-}4$ for MNIST, $9\text{e-}4$ for CIFAR-10, and $7\text{e-}4$ for SVHN.
}
}
\label{Fig::Hidden_Units}
\end{figure}

\subsection{Network Width} 
As observed by several recent studies \citep[see e.g.,][]{neyshabur2014search,jiang2019fantastic}, wider networks can lead to a smaller generalization gap. This may seem contradictory to the traditional wisdom as one may expect that a class of wider networks has a higher VC-dimension and, hence, would have a higher generalize gap. In our experiment, we use the SGLD algorithm to train neural networks with different widths. The training process runs for $400$ epochs until the training accuracy is $1.0$. We compare our generalization bound with the generalization gap in Figure~\ref{Fig::Hidden_Units}. As shown, both the generalization gap and our bound are decreasing with respect to the network width.

\begin{table}
\small\centering
\resizebox{0.85\textwidth}{!}{
\renewcommand{\arraystretch}{1.5}
\begin{tabular}{llcccccc}
\toprule
\texttt{data set} & \texttt{method} & \texttt{lr} & \texttt{width} & \texttt{depth} & \texttt{$\tau$} & \texttt{$\Psi$} & \texttt{MI} \\
\toprule
\multirow{2}{*}{\textsc{MNIST}} & \textsc{Ours (Proposition~\ref{prop::gen_bound_SGLD})} & \textbf{0.70} & \textbf{1.00} & \textbf{0.56} & \textbf{0.50} & \textbf{0.75} & \textbf{0.34}\\
\cmidrule{2-8}
& \textsc{\citet{negrea2019information}} & 0.26 & 0.26 & 0.48 & 0.25 & 0.33 & 0.12\\
\midrule
\multirow{2}{*}{\textsc{CIFAR-10}}&\textsc{Ours (Proposition~\ref{prop::gen_bound_SGLD})} & \textbf{0.41} & \textbf{0.93} & \textbf{1.00} & \textbf{0.45} & \textbf{0.78} & \textbf{0.25}\\
\cmidrule{2-8}
&\textsc{\citet{negrea2019information}} & 0.33 & 0.41 &0.85 & 0.38 & 0.53 & 0.16\\
\bottomrule 
\end{tabular}
%\\[10pt]
}
\caption{\small{We adopt the three evaluation criteria proposed in \citet{jiang2019fantastic} for comparing our generalization bound with the benchmark method \citep[Theorem~3.1 of][]{negrea2019information}:  (i) Kendall's rank-correlation coefficient ($\tau$), (ii) Granulated Kendall's coefficient ($\Psi$), and (iii) conditional independent test (MI). 
All scores, except MI, are within $[-1,1]$ and the score of MI is normalized to $[0,1]$. 
We also report the correlations when a single hyper-parameter (e.g., learning rate (lr)) is varying.
}
}
\label{table:comparison}
\end{table}

\subsection{Comparison with Benchmarks} 
To evaluate our bound, we adopt the three criteria proposed in \citet{jiang2019fantastic}: (i) Kendall's rank-correlation coefficient ($\tau$) \citep{kendall1938new}, (ii) Granulated Kendall's coefficient ($\Psi$), and (iii) conditional independent test via mutual information (MI) \citep{verma1991equivalence}. 
In our experiments, we select $3$ commonly used hyper-parameters (i.e., learning rate (lr), width, depth), which are believed to influence the generalization gap, and let each hyper-parameter choose three different values. 
We train $27$ neural networks under all combinations of hyper-parameters and assess the correlations between the generalization bound and the generalization gap.

We compare our generalization bound with the gradient-prediction-residual bound in Theorem~3.1 of \citet{negrea2019information} under the above three evaluation criteria. As shown in Table~\ref{table:comparison}, our generalization bound is highly correlated with the true generalization gap and outperforms the benchmark under all the criteria suggested in \citet{jiang2019fantastic}.

\section{Conclusion}

In this paper, we investigate the generalization of noisy iterative algorithms and derive three generalization bounds based on different $f$-divergence. 
We establish a unified framework and leverage fundamental tools from information theory (e.g., strong data processing inequalities and properties of additive noise channels) for proving these bounds. 
We demonstrate our generalization bounds through applications, including DP-SGD, FL, and SGLD, in which our bounds own a simple form and can be estimated from data. Numerical experiments suggest that our bounds can predict the behavior of the true generalization gap.

\section*{Acknowledgments}
The authors would like to thank the anonymous reviewers and the action editor for their careful reading of our paper and their valuable suggestions. 
The work of H. Wang and F. P. Calmon is supported in part by the National Science Foundation under grants CAREER 1845852, IIS 1926925, and FAI 2040880 and F. P. Calmon also acknowledges a Google Faculty Research Award and an Amazon Research Award.

\clearpage
\appendix
\section{Proofs for Section~\ref{sec::Prelim}}

\subsection{Proof of Lemma~\ref{lem::Bu20}}
\label{append::proof_Bu20}

Inequalities~\eqref{eq::T_inf_bound} and \eqref{eq::chi_inf_bound} can be proved by combining the proof technique introduced in \citet{bu2020tightening} and variational representations of $f$-divergence \citep{nguyen2010estimating}. Note that these two bounds can also be obtained from Corollary~1 in \citet{rodriguez2021tighter} and applying Jensen's inequality to Eq.~(15) in the same paper, respectively. We provide the proof here for the sake of completeness. 
\begin{proof}
Recall that \citep[see Example~6.1~and~6.4 in][]{wu2017lecture} for any two probability distributions $P$ and $Q$ over a set $\mathcal{X}\subseteq \Reals^d$ and a constant $A>0$,
\begin{align}
    \TV(P\|Q)
    &= \sup_{\substack{h:\mathcal{X}\to \Reals \\ 0\leq \|h\|_{\infty} \leq A}} \left|\frac{\EEE{P}{h(\rX)} - \EEE{Q}{h(\rX)}}{A} \right|, \label{eq::TV_vr}\\
    \chisquare(P\|Q)
    &= \sup_{h:\mathcal{X}\to \Reals} \frac{(\EEE{P}{h(\rX)} - \EEE{Q}{h(\rX)})^2}{\Varr{Q}{h(\rX)}}. \label{eq::chi_vr}
\end{align}
On the other hand, the expected generalization gap can be written as
\begin{align*}
    \EE{L_\mu(\rW)-L_{\rS}(\rW)}
    = \frac{1}{n} \sum_{i=1}^n \left(\EE{\ell(\rW,\bar{\rZ}_i)} - \EE{\ell(\rW,\rZ_i)} \right),
\end{align*}
where $(\rW,\bar{\rZ}_i) \sim P_{\rW}\otimes \mu$. Consequently,
\begin{align*}
    \left|\EE{L_\mu(\rW)-L_{\rS}(\rW)}\right|
    \leq \frac{1}{n} \sum_{i=1}^n \left|\EE{\ell(\rW,\bar{\rZ}_i)} - \EE{\ell(\rW,\rZ_i)} \right|.
\end{align*}
If the loss function is upper bounded by $A>0$, taking $P=P_{\rW,\rZ_i}$ and $Q=P_{\rW} \otimes \mu$ in \eqref{eq::TV_vr} yields 
\begin{align*}
    \left|\EE{L_\mu(\rW)-L_{\rS}(\rW)}\right|
    \leq \frac{A}{n} \sum_{i=1}^n \TV(P_{\rW,\rZ_i}\| P_{\rW} \otimes \mu)
    = \frac{A}{n} \sum_{i=1}^n {\textnormal T}(\rW;\rZ_i).
\end{align*}
Similarly, taking $P=P_{\rW,\rZ_i}$ and $Q=P_{\rW} \otimes \mu$ in \eqref{eq::chi_vr} leads to
\begin{align*}
    \left|\EE{L_\mu(\rW)-L_{\rS}(\rW)}\right|
    \leq \frac{1}{n} \sum_{i=1}^n \sqrt{\Var{\ell(\rW;\rZ)} \cdot \chi^2(\rW;\rZ_i)}
\end{align*}
where $(\rW,\rZ_i) \sim P_{\rW,\rZ_i}$ and $(\rW,\rZ)\sim P_{\rW}\otimes \mu$.
\end{proof}

\section{Proofs for Section~\ref{sec::prop_add_channel}}

\subsection{Proof of Lemma~\ref{lem::gen_HWI}}
\label{append::gen_HWI}

Lemma~\ref{lem::gen_HWI} follows from a slight tweak of the proof of Theorem~4 in \citet{polyanskiy2016dissipation}.
\begin{proof}
First, we choose a coupling $P_{\rX,\rX'}$, which has marginals $P_{\rX}$ and $P_{\rX'}$. The probability distribution of $\rX+m\rN$ can be written as the convolution of $P_{\rX}$ and $P_{m\rN}$. Specifically, for any measurable set $\mathcal{A}\subseteq \mathcal{X}$,
\begin{align*}
    P_{\rX+m\rN}(\mathcal{A})
    = \int_{\by\in\mathcal{A}} \int_{\bx \in \mathcal{X}} \dif P_{m\rN}(\by - \bx) \dif P_{\rX}(\bx)
    = \int_{\by\in\mathcal{A}} \int_{\bx \in \mathcal{X}} \int_{\bx' \in \mathcal{X}} \dif P_{\bx+m\rN}(\by) \dif P_{\rX,\rX'}(\bx,\bx').
\end{align*}
Similarly, we have
\begin{align*}
    P_{\rX'+m\rN}(\mathcal{A})
    &= \int_{\by\in\mathcal{A}} \int_{\bx \in \mathcal{X}} \int_{\bx' \in \mathcal{X}} \dif P_{\bx'+m\rN}(\by) \dif P_{\rX,\rX'}(\bx,\bx').
\end{align*}
Since the mapping $(P,Q) \to \textnormal{D}_{f}(P\|Q)$ is convex \citep[see Theorem~6.1 in][for a proof]{polyanskiy2014lecture}, Jensen's inequality yields
\begin{align}
    \textnormal{D}_{f}\left(P_{\rX+m\rN} \| P_{\rX'+m\rN} \right)
    &\leq \int_{\bx \in \mathcal{X}} \int_{\bx' \in \mathcal{X}} \textnormal{D}_{f}(P_{\bx+m\rN}\|P_{\bx'+m\rN}) \dif P_{\rX,\rX'}(\bx,\bx') \nonumber\\
    &= \EE{\mathsf{C}_f(\rX,\rX';m)},\label{eq::KL_ub_Exp_Ct}
\end{align}
where the last step follows from the definition. The left-hand side of \eqref{eq::KL_ub_Exp_Ct} only relies on the marginal distributions of $\rX$ and $\rX'$, so taking the infimum on both sides of \eqref{eq::KL_ub_Exp_Ct} over all couplings $P_{\rX,\rX'}$ leads to the desired conclusion.
\end{proof}

\subsection{A Useful Property}

We derive a useful property of $\mathsf{C}_f(\bx,\by;m)$ which will be used in our proofs.
\begin{lemma}
\label{lem::prop_Ct}
For any $\bz\in \Reals^d$ and $a>0$, the function $\mathsf{C}_f(\bx,\bx';m)$ in \eqref{eq::cost_func} satisfies
\begin{align*}
    &\mathsf{C}_f(a\bx + \bz, a\bx' + \bz; m) 
    = \mathsf{C}_f\left(\bx, \bx'; \frac{m}{a}\right), \quad \mathsf{C}_f(\bx, \bx'; m) = \mathsf{C}_f(-\bx', -\bx; m).
\end{align*}
\end{lemma}
\begin{proof}
For simplicity, we assume that $\rN$ is a continuous random variable in $\Reals^d$ with probability density function (PDF) $p(\bw)$. Then the PDFs of $a\bx + \bz + m\rN$ and $a\bx' + \bz + m\rN$ are 
\begin{align*}
    \frac{1}{m^d}\cdot p\left(\frac{\bw - a\bx - \bz}{m}\right) \quad \text{and}\quad \frac{1}{m^d}\cdot p\left(\frac{\bw - a\bx' - \bz}{m}\right).
\end{align*}
By definition,
\begin{align}
    \mathsf{C}_f(a\bx + \bz, a\bx' + \bz; m)
    &= \textnormal{D}_{f}(P_{a\bx + \bz + m\rN}\|P_{a\bx' + \bz + m\rN}) \nonumber\\
    &= \frac{1}{m^d} \int_{\Reals^d} p\left(\frac{\bw - a\bx' - \bz}{m}\right) f\left(\frac{p\left(\frac{\bw - a\bx - \bz}{m}\right)}{p\left(\frac{\bw - a\bx' - \bz}{m}\right)}\right) \dif \bw. \label{eq::KL_int_density}
\end{align}
Let $\bv = (\bw - \bz)/a$. Then \eqref{eq::KL_int_density} is equal to
\begin{align*}
    \frac{a^d}{m^d} \int_{\Reals^d} p\left(\frac{\bv - \bx'}{m/a}\right) f\left(\frac{p\left(\frac{\bv - \bx}{m/a}\right)}{p\left(\frac{\bv - \bx'}{m/a}\right)}\right) \dif \bv
    = \mathsf{C}_f\left(\bx, \bx'; \frac{m}{a}\right).
\end{align*}
Therefore, $\mathsf{C}_f(a\bx + \bz, a\bx' + \bz; m) = \mathsf{C}_f\left(\bx, \bx'; \frac{m}{a}\right)$. By choosing $a=1$ and $\bz = -\bx - \bx'$, we have $\mathsf{C}_f(-\bx', -\bx; m) = \mathsf{C}(\bx, \bx'; m)$. 
\end{proof}

\subsection{Proof of Table~\ref{table:C_delta_exp}}
\label{append::table_C_delta}

We first derive closed-form expressions (or upper bounds) for the function $\delta(A,m)$. The closed-form expression of $\delta(A,m)$ for uniform distribution can be naturally obtained from its definition so we skip the proof. The closed-form expressions for standard multivariate Gaussian distribution and standard univariate Laplace distribution can be found at \citet{polyanskiy2016dissipation} and \citet{asoodeh2020privacy}. 
In what follows, we provide an upper bound for $\delta(A,m)$ when $\rN$ follows a standard multivariate Laplace distribution. 
\begin{proof}
For a given positive number $A$ and a random variable $\rN$ which follows a standard multivariate Laplace distribution, consider the following optimization problem:
\begin{align*}
    \sup_{\|\bv\|_1 \leq A} \TV\left(P_{\rN}\| P_{\bv + \rN}\right)
    = \sup_{\|\bv\|_1 \leq A} \EE{\left(1-\frac{\exp(-\|\rN-\bv\|_1)}{\exp(-\|\rN\|_1)} \right)\indicator{\|\rN-\bv\|_1 \geq \|\rN\|_1}}.
\end{align*}
By exchanging the supremum and the expectation, we have 
\begin{align}
\label{eq::sup_TV_ub_exp_sup}
    \sup_{\|\bv\|_1 \leq A} \TV\left(P_{\rN}\| P_{\bv + \rN}\right)
    \leq \EE{\sup_{\|\bv\|_1 \leq A} \left\{ \left(1-\frac{\exp(-\|\rN-\bv\|_1)}{\exp(-\|\rN\|_1)} \right)\indicator{\|\rN-\bv\|_1 \geq \|\rN\|_1} \right\}}.
\end{align}
Note that 
\begin{align*}
    &\sup_{\|\bv\|_1 \leq A} \left\{ \left(1-\frac{\exp(-\|\rN-\bv\|_1)}{\exp(-\|\rN\|_1)} \right)\indicator{\|\rN-\bv\|_1 \geq \|\rN\|_1} \right\}\\
    &= 1 - \exp\left(-\sup_{\substack{\|\bv\|_1 \leq A}} \left\{\|\rN-\bv\|_1 - \|\rN\|_1\right\} \right)
    = 1-\exp(-A).
\end{align*}
Substituting this equality into \eqref{eq::sup_TV_ub_exp_sup} gives
\begin{align*}
    \sup_{\|\bx-\bx'\|_1 \leq A} \TV\left(P_{\bx+\rN}\| P_{\bx' + \rN}\right)
    =\sup_{\|\bv\|_1 \leq A} \TV\left(P_{\rN}\| P_{\bv + \rN}\right)
    \leq 1-\exp(-A),
\end{align*}
which leads to $\delta(A,1) \leq 1-\exp(-A)$. Finally, we have 
\begin{align*}
    \delta(A, m) 
    = \delta\left(\frac{A}{m}, 1\right) 
    \leq 1-\exp\left(-\frac{A}{m}\right).
\end{align*}
\end{proof}

Now we consider the function $\mathsf{C}_{f}(\bx, \bx'; m)$.
\begin{proof}
By Lemma~\ref{lem::prop_Ct}, we have
\begin{align*}
    \mathsf{C}_{\scalebox{.6}{\textnormal KL}}(\bx, \bx'; m)
    = \mathsf{C}_{\scalebox{.6}{\textnormal KL}}\left(0, \frac{\bx' - \bx}{m}; 1\right).
\end{align*}
We denote $(\bx' - \bx)/m$ by $\bv$. Since all the coordinates of $\rN = (\rN_1\cdots,\rN_d)$ are mutually independent, $P_{\rN} = P_{\rN_1} \cdots P_{\rN_d}$ and $P_{\bv+\rN} = P_{v_1 + \rN_1} \cdots P_{v_d+\rN_d}$. By the chain rule of KL-divergence, we have
\begin{align}
\label{eq::exp_C_sum_KL}
    \mathsf{C}_{\scalebox{.6}{\textnormal KL}}(\bx, \bx'; m) 
    = \KL(P_{\rN} \| P_{\bv+\rN})
    = \sum_{i=1}^d \KL(P_{\rN_i} \| P_{v_i+\rN_i}).
\end{align}
Hence, we only need to calculate $\KL(P_{\rN} \| P_{v+\rN})$ for a constant $v\in\Reals$ and a random variable $\rN \in \Reals$. \\
(1) If $\rN$ follows a standard Gaussian distribution, then
\begin{align*}
    \KL(P_{\rN} \| P_{v+\rN}) 
    &= \EE{\log \frac{\exp(-\rN^2/2)}{\exp(-(\rN-v)^2/2)}}\\
    &= \frac{1}{2}\EE{(\rN-v)^2 - \rN^2}
    = \frac{v^2}{2}.
\end{align*}
Substituting this equality into \eqref{eq::exp_C_sum_KL} gives
\begin{align*}
    \mathsf{C}_{\scalebox{.6}{\textnormal KL}}(\bx, \bx'; m)
    = \frac{\|\bv\|_2^2}{2}
    = \frac{\|\bx-\bx'\|_2^2}{2m^2},
\end{align*}
where the last step is due to the definition of $\bv$.\\
(2) If $\rN$ follows a standard Laplace distribution, then
\begin{align*}
    \KL(P_{\rN} \| P_{v+\rN}) 
    = \EE{\log \frac{\exp(-|\rN|)}{\exp(-|\rN-v|)}}
    = |v| + \exp(-|v|) - 1.
\end{align*}
Substituting this equality into \eqref{eq::exp_C_sum_KL} gives
\begin{align*}
    \mathsf{C}_{\scalebox{.6}{\textnormal KL}}(\bx, \bx'; m)
    &= \sum_{i=1}^d |v_i| + \exp(-|v_i|) - 1
    = \frac{\|\bx-\bx'\|_1}{m} + \sum_{i=1}^d \left(\exp\left(-\frac{|x_i-x'_i|}{m}\right) - 1 \right)\\
    &\leq \frac{\|\bx-\bx'\|_1}{m}.
\end{align*}
Similarly, by Lemma~\ref{lem::prop_Ct}, we have
\begin{align*}
    \mathsf{C}_{\chi^2}(\bx, \bx'; m)
    = \mathsf{C}_{\chi^2}\left(0, \frac{\bx' - \bx}{m}; 1\right).
\end{align*}
We denote $(\bx' - \bx)/m$ by $\bv$. By the property of $\chi^2$-divergence \citep[see Section~2.4 in][]{tsybakov2009introduction}, we have
\begin{align}
\label{eq::chi_tensor}
    \mathsf{C}_{\chi^2}(\bx, \bx'; m)
    = \chisquare(P_{\rN}\|P_{\bv+\rN})
    = \prod_{i=1}^d \left(1 + \chisquare(P_{\rN_i}\| P_{v_i+\rN_i})\right) - 1.
\end{align}
Hence, we only need to calculate $\chisquare(P_{\rN} \| P_{v+\rN})$ for $v\in\Reals$ and $\rN \in \Reals$.\\
(1) If $\rN$ follows a standard Gaussian distribution, then
\begin{align*}
    \chisquare(P_{\rN}\|P_{v+\rN})
    &= \EE{\frac{\exp(-\rN^2/2)}{\exp(-(\rN-v)^2/2)}} -1\\
    &= \exp(v^2/2) \EE{\exp(-v N)} - 1
    = \exp(v^2) - 1.
\end{align*}
Substituting this equality into \eqref{eq::chi_tensor} gives
\begin{align*}
    \mathsf{C}_{\chi^2}(\bx, \bx'; m)
    = \exp(\|\bv\|_2^2) - 1
    = \exp\left( \frac{\|\bx-\bx'\|_2^2}{m^2}\right) - 1.
\end{align*}
(2) If $\rN$ follows a standard Laplace distribution, then \begin{align*}
    \chisquare(P_{\rN}\|P_{v+\rN})
    &= \EE{\frac{\exp(-|\rN|)}{\exp(-|\rN-v|)}} -1\\
    &= \frac{2}{3}\exp(|v|) + \frac{1}{3}\exp(-2|v|) -1.
\end{align*}
Substituting this equality into \eqref{eq::chi_tensor} gives
\begin{align*}
    \mathsf{C}_{\chi^2}(\bx, \bx'; m)
    &= \prod_{i=1}^d \left(\frac{2}{3}\exp\left(\frac{|x_i - x'_i|}{m}\right) + \frac{1}{3}\exp\left(\frac{-2|x_i - x'_i|}{m}\right)\right) -1\\
    &\leq \exp\left(\frac{\|\bx - \bx'\|_1}{m}\right) - 1.
\end{align*}
Finally, we use Pinsker's inequality \citep[see Theorem 4.5 in][for a proof]{wu2017lecture} for proving an upper bound of $\mathsf{C}_{\scalebox{.6}{\textnormal TV}}(\bx, \bx'; m)$:
\begin{align*}
    \mathsf{C}_{\scalebox{.6}{\textnormal TV}}(\bx, \bx'; m)
    &= \TV\left(P_{\bx + m\rN} \| P_{\bx' + m\rN}\right)\\
    &\leq \sqrt{\frac{\KL\left(P_{\bx + m\rN} \| P_{\bx' + m\rN}\right)}{2}}\\
    &= \sqrt{\frac{\mathsf{C}_{\scalebox{.6}{\textnormal KL}}(\bx, \bx'; m)}{2}}.
\end{align*}
Hence, any upper bound of $\mathsf{C}_{\scalebox{.6}{\textnormal KL}}(\bx, \bx'; m)$ can be naturally translated into an upper bound for $\mathsf{C}_{\scalebox{.6}{\textnormal TV}}(\bx, \bx'; m)$. This is how we obtain the upper bounds of $\mathsf{C}_{\scalebox{.6}{\textnormal TV}}(\bx, \bx'; m)$ under Gaussian or Laplace distribution in Table~\ref{table:C_delta_exp}. On the other hand, if $\rN$ follows a uniform distribution on $[-1,1]\subseteq \Reals$, by Lemma~\ref{lem::prop_Ct} we have
\begin{align*}
     \mathsf{C}_{\scalebox{.6}{\textnormal TV}}(x, x'; m)
     =\mathsf{C}_{\scalebox{.6}{\textnormal TV}}\left(0, \frac{x' - x}{m}; 1\right)
     = \min\left\{1,\left|\frac{x-x'}{2m}\right|\right\}.
\end{align*}
Note that in this case $\bx,\bx' \in \Reals$ so we write them as $x,x'$.
\end{proof}
\begin{remark}
We used Pinsker's inequality for deriving an upper bound of $\mathsf{C}_{\scalebox{.6}{\textnormal TV}}(\bx, \bx'; m)$ in the above proof. One can potentially tighten this bound by exploring other $f$-divergence inequalities \citep[see e.g., Eq.~4 in][]{sason2016f}.
\end{remark}

\section{Proofs for Section~\ref{sec::Gen_Bound}}
\subsection{Proof of Theorem~\ref{thm::gen_bound_general_noisy_alg}}
\label{append::gen_bound_general_noisy_alg}

\begin{proof}
Combining Lemma~\ref{lem::Noisy_ite_alg_SDPI} and \ref{lem::ub_MI_WT_Zi_general} together leads to an upper bound of $I_f(\rW_T; \rZ_i)$ for any data point $\rZ_i$ used at the $t$-th iteration:
\begin{align}
\label{eq::f_inf_ub_proof}
    I_f(\rW_T; \rZ_i)
    \leq \EE{\mathsf{C}_f\left(g(\rW_{t-1}, \rZ), g(\rW_{t-1},\bar{\rZ}); \frac{m_tb_t}{\eta_t}\right)} \cdot \prod_{t'=t+1}^T \delta(D+2\eta_{t'}K, m_{t'}).
\end{align}
Additionally, if the loss $\ell(\bw,\rZ)$ is $\sigma$-sub-Gaussian for all $\bw \in \mathcal{W}$, Lemma~\ref{lem::Bu20} and Assumption~\ref{assump::index_w/o_rep} (Sampling w/o Replacement) altogether yield
\begin{align}
    \left|\EE{L_\mu(\rW_T)-L_{\rS}(\rW_T)}\right|
    &\leq \frac{1}{n}\sum_{i=1}^n \sqrt{2\sigma^2 I(\rW_T; \rZ_i)} \nonumber\\
    &= \frac{1}{n}\sum_{t=1}^T \sum_{i\in \mathcal{B}_t} \sqrt{2\sigma^2 I(\rW_T; \rZ_i)}. \label{eq::gen_ub_cond_U_1_a}
\end{align}
Substituting \eqref{eq::f_inf_ub_proof} into \eqref{eq::gen_ub_cond_U_1_a} yields the following upper bound of the expected generalization gap:
\begin{align*}
    &\frac{\sqrt{2}\sigma}{n} \sum_{t=1}^T \sum_{i\in \mathcal{B}_t} \sqrt{\EE{\mathsf{C}_{\scalebox{.6}{\textnormal KL}}\left(g(\rW_{t-1}, \rZ), g(\rW_{t-1},\bar{\rZ}); \frac{m_tb_t}{\eta_t}\right)} \cdot \prod_{t'=t+1}^T \delta(D+2\eta_{t'}K, m_{t'})}\\
    &=\frac{\sqrt{2}\sigma}{n} \sum_{t=1}^T b_t \sqrt{\EE{\mathsf{C}_{\scalebox{.6}{\textnormal KL}}\left(g(\rW_{t-1}, \rZ), g(\rW_{t-1},\bar{\rZ}); \frac{m_tb_t}{\eta_t}\right)} \cdot \prod_{t'=t+1}^T \delta(D+2\eta_{t'}K, m_{t'})}.
\end{align*}
Similarly, we can obtain another two generalization bounds using Lemma~\ref{lem::Bu20} and the upper bound in \eqref{eq::f_inf_ub_proof}.
\end{proof}

\section{Proofs for Section~\ref{sec::application}}

\subsection{Proof of Proposition \ref{prop::DP_SGD_Lap} and \ref{prop::DP_SGD_Gaussian}}
\label{append::gen_bound_DP_SGD}

In the setting of DP-SGD, the three generalization bounds in Theorem~\ref{thm::gen_bound_general_noisy_alg} become
\begin{align}
\label{eq::KL_bound_DP_SGD}
    \frac{\sqrt{2}\sigma}{n} \sum_{t=1}^T b_t \sqrt{ \EE{\mathsf{C}_{\scalebox{.6}{\textnormal KL}}\left(g(\rW_{t-1}, \rZ), g(\rW_{t-1},\bar{\rZ}); b_t \right)} \cdot \left(\delta(D+2\eta K, \eta) \right)^{T-t}}
\end{align}
where $\sigma$ is the sub-Gaussian constant; 
\begin{align}
\label{eq::TV_bound_DP_SGD}
    \frac{A}{n} \sum_{t=1}^T b_t \EE{\mathsf{C}_{\scalebox{.6}{\textnormal TV}}\left(g(\rW_{t-1}, \rZ), g(\rW_{t-1},\bar{\rZ}); b_t\right)} \cdot \left(\delta(D+2\eta K, \eta)\right)^{T-t}
\end{align}
where $A$ is an upper bound of the loss function; and
\begin{align}
\label{eq::chi_bound_DP_SGD}
    \frac{\sigma}{n} \sum_{t=1}^T b_t \sqrt{\EE{\mathsf{C}_{\chi^2}\left(g(\rW_{t-1}, \rZ), g(\rW_{t-1},\bar{\rZ}); b_t\right)} \cdot \left(\delta(D+2\eta K, \eta) \right)^{T-t}}
\end{align}
where $\sigma \defined \sqrt{\Var{\ell(\rW_T;\rZ)}}$. 
\begin{proof}
We prove Proposition~\ref{prop::DP_SGD_Gaussian} first.
\begin{itemize}[leftmargin=*]
\item If the additive noise follows a standard multivariate Gaussian distribution, Table~\ref{table:C_delta_exp} shows that 
\begin{align}
    \delta(D+2\eta K, \eta) 
    &= 1-2\bar{\Phi}\left(\frac{D+2\eta K}{2\eta}\right) \label{eq::delta_gaussian_prf_1},\\
    \EE{\mathsf{C}_{\scalebox{.6}{\textnormal KL}}\left(g(\rW_{t-1}, \rZ), g(\rW_{t-1}, \bar{\rZ}); b_t\right)}
    &= \frac{1}{2 b_t^2} \EE{\|g(\rW_{t-1}, \rZ) - g(\rW_{t-1}, \bar{\rZ})\|_2^2}. \label{eq::ub_C_dif_l2_g}
\end{align}
We introduce a constant vector $\be$ whose value will be specified later. Since $\|\ba-\bb\|_2^2\leq 2\|\ba\|_2^2 + 2\|\bb\|_2^2$, we have
\begin{align}
    &\frac{1}{2b_t^2}\EE{\|g(\rW_{t-1}, \rZ) - g(\rW_{t-1}, \bar{\rZ})\|_2^2} \nonumber\\
    &\leq \frac{1}{b_t^2} \left(\EE{\|g(\rW_{t-1}, \rZ) - \be\|_2^2} + \EE{\|g(\rW_{t-1}, \bar{\rZ}) - \be\|_2^2}\right) \nonumber\\
    &= \frac{2}{b_t^2} \EE{\|g(\rW_{t-1}, \rZ) - \be\|_2^2}, \label{eq::ub_l2_dif_g}
\end{align}
where the last step is because $\rW_{t-1}$ is independent of $(\rZ, \bar{\rZ})$ and $\rZ, \bar{\rZ}$ follow the same distribution. By choosing the constant vector $\be = \EE{g(\rW_{t-1}, \rZ)}$, we have
\begin{align}
\label{eq::ub_l2_dif_g_Var_prf}
    \EE{\|g(\rW_{t-1}, \rZ) - \be\|_2^2}
    = \Var{g(\rW_{t-1}, \rZ)}.
\end{align}
Combining (\ref{eq::ub_C_dif_l2_g}--\ref{eq::ub_l2_dif_g_Var_prf}) gives
\begin{align}
\label{eq::exp_C_ub_var}
    \EE{\mathsf{C}_{\scalebox{.6}{\textnormal KL}}\left(g(\rW_{t-1}, \rZ), g(\rW_{t-1}, \bar{\rZ}); b_t\right)}
    \leq \frac{2}{b_t^2} \Var{g(\rW_{t-1}, \rZ)}.
\end{align}
Substituting \eqref{eq::delta_gaussian_prf_1}, \eqref{eq::exp_C_ub_var} into \eqref{eq::KL_bound_DP_SGD} leads to the generalization bound in \eqref{eq::gen_bound_KL_Gauss}. \\

\item Similarly, Table~\ref{table:C_delta_exp} shows for Gaussian noise
\begin{align}
    \EE{\mathsf{C}_{\scalebox{.6}{\textnormal TV}}\left(g(\rW_{t-1}, \rZ), g(\rW_{t-1}, \bar{\rZ}); b_t\right)}
    &\leq \frac{1}{2 b_t} \EE{\|g(\rW_{t-1}, \rZ) - g(\rW_{t-1}, \bar{\rZ})\|_2}. \label{eq::C_TV_Tab_DP}
\end{align}
Furthermore, by the triangle inequality, 
\begin{align}
    &\frac{1}{2b_t} \EE{\|g(\rW_{t-1}, \rZ) - g(\rW_{t-1}, \bar{\rZ})\|_2} \nonumber\\
    &\leq \frac{1}{2b_t} \left(\EE{\|g(\rW_{t-1}, \rZ) - \be\|_2} + \EE{\|g(\rW_{t-1}, \bar{\rZ}) - \be\|_2}\right) \nonumber\\
    &= \frac{1}{b_t} \EE{\|g(\rW_{t-1}, \rZ) - \be\|_2}. \label{eq::C_TV_Tab_DP_tri}
\end{align}
By choosing the constant vector $\be = \EE{g(\rW_{t-1}, \rZ)}$ and combining \eqref{eq::C_TV_Tab_DP} with \eqref{eq::C_TV_Tab_DP_tri}, we have 
\begin{align}
\label{eq::exp_C_ub_L2}
    \EE{\mathsf{C}_{\scalebox{.6}{\textnormal TV}}\left(g(\rW_{t-1}, \rZ), g(\rW_{t-1}, \bar{\rZ}); b_t\right)}
    &\leq \frac{1}{b_t}\EE{\|g(\rW_{t-1}, \rZ) - \be\|_2}.
\end{align}
Substituting \eqref{eq::delta_gaussian_prf_1}, \eqref{eq::exp_C_ub_L2} into \eqref{eq::TV_bound_DP_SGD} leads to the generalization bound in \eqref{eq::gen_bound_TV_Gauss}.

\item Finally, Table~\ref{table:C_delta_exp} shows for Gaussian noise
\begin{align}
    \EE{\mathsf{C}_{\chi^2}\left(g(\rW_{t-1}, \rZ), g(\rW_{t-1}, \bar{\rZ}); b_t\right)}
    &= \EE{\exp\left(\frac{\|g(\rW_{t-1}, \rZ) - g(\rW_{t-1}, \bar{\rZ})\|_2^2}{b_t^2} \right)} - 1. \label{eq::C_chi_Tab_DP} 
\end{align}
The Cauchy-Schwarz inequality implies that
\begin{align}
    &\EE{\exp\left(\frac{\|g(\rW_{t-1}, \rZ) - g(\rW_{t-1}, \bar{\rZ})\|_2^2}{b_t^2} \right)} \nonumber \\
    &\leq \EE{\exp\left(\frac{2\|g(\rW_{t-1}, \rZ) - \be\|_2^2}{b_t^2}\right) \exp\left(\frac{2\|g(\rW_{t-1}, \bar{\rZ}) - \be\|_2^2}{b_t^2}\right)} \nonumber\\
    &\leq \sqrt{\EE{\exp\left(\frac{4\|g(\rW_{t-1}, \rZ) - \be\|_2^2}{b_t^2}\right)} \EE{\exp\left(\frac{4\|g(\rW_{t-1}, \bar{\rZ}) - \be\|_2^2}{b_t^2}\right)}} \nonumber\\
    &= \EE{\exp\left(\frac{4\|g(\rW_{t-1}, \rZ) - \be\|_2^2}{b_t^2}\right)}. \label{eq::C_chi_ub_DP} 
\end{align}
By choosing the constant vector $\be = \EE{g(\rW_{t-1}, \rZ)}$ and combining \eqref{eq::C_chi_Tab_DP} with \eqref{eq::C_chi_ub_DP}, we have 
\begin{align}
\label{eq::exp_C_ub_L22}
    \EE{\mathsf{C}_{\chi^2}\left(g(\rW_{t-1}, \rZ), g(\rW_{t-1}, \bar{\rZ}); b_t\right)}
    &\leq \EE{\exp\left(\frac{4\|g(\rW_{t-1}, \rZ) - \be\|_2^2}{b_t^2}\right)} - 1.
\end{align}
Since for any $x\geq 0$ and $b\geq 1$,
\begin{align*}
    \exp\left(\frac{x}{b}\right) - 1
    \leq \frac{\exp(x) - 1}{b},
\end{align*}
the inequality in \eqref{eq::exp_C_ub_L22} can be further upper bounded as
\begin{align}
\label{eq::exp_C_ub_L22_fur_boud}
    \EE{\mathsf{C}_{\chi^2}\left(g(\rW_{t-1}, \rZ), g(\rW_{t-1}, \bar{\rZ}); b_t\right)}
    &\leq \frac{1}{b_t^2} \left(\EE{\exp\left(4\|g(\rW_{t-1}, \rZ) - \be\|_2^2\right)} - 1 \right).
\end{align}
Substituting \eqref{eq::delta_gaussian_prf_1}, \eqref{eq::exp_C_ub_L22_fur_boud} into \eqref{eq::chi_bound_DP_SGD} leads to the generalization bound in \eqref{eq::gen_bound_chi_Gauss}. 
\end{itemize}
\end{proof}
By a similar analysis, we can prove the generalization bounds in Proposition~\ref{prop::DP_SGD_Lap} for the Laplace mechanism.

\subsection{Proof of Proposition~\ref{prop::gen_bound_FL}}
\label{append::gen_bound_FL}

\begin{proof}
Within the $t$-th global update, we can rewrite the local updates conducted by the client $l\in \mathcal{S}_t$ as follows. The parameter is initialized by $\rW_{t,0}^l = \rW_{t-1}$ and for $j\in [M]$,
\begin{subequations}
\begin{align}
    &\rU_{t,j}^{l} = \rW_{t,j-1}^{l} - \eta \cdot g\left(\rW_{t,j-1}^{l}, \{\rZ_i^l\}_{i\in [b]}\right)\\
    &\rV_{t,j}^{l} = \rU_{t,j}^{l} + \eta \cdot \rN_{t,j}^{l} \\
    &\rW_{t,j}^{l} = \mathsf{Proj}_{\mathcal{W}} \left(\rV_{t,j}^{l} \right)
\end{align}
\end{subequations}
where $\{\rZ_i^l\}_{i\in [b]}$ are drawn independently from the data distribution $\mu_l$ and $\rN_{t,j}^{l} \sim N(0, \mathbf{I}_d)$. If a data point $\rZ_{i}^k$ is used at the $t$-th global update, $j$-th local update, by the client $k\in\mathcal{S}$, then the following Markov chain holds:
\begin{align*}
    \underbrace{\rZ_i^k 
    \to \{\rU_{t,j}^{l}\}_{l\in \mathcal{S}_t}
    \to \{\rV_{t,j}^{l}\}_{l\in \mathcal{S}_t} 
    \to \{\rW_{t,j}^{l}\}_{l\in \mathcal{S}_t} 
    \to \cdots 
    \to \{\rW_{t,M}^{l}\}_{l\in \mathcal{S}_t}}_{\text{local}} 
    \underbrace{\to \rW_{t}}_{\text{global}}
    \to \cdots
    \to \rW_T
\end{align*}
Let $\mathcal{U}_{t,j}^l$ be the range of $\rU_{t,j}^l$. Note that $\mathsf{diam}\left(\mathcal{U}_{t,j}^l\right) \leq \mathsf{diam}(\mathcal{W}) + 2\eta K= D+2\eta K$. Since $|\mathcal{S}_t| = C$, then 
\begin{align*}
    \mathsf{diam}\left(\prod_{l\in\mathcal{S}_t} \mathcal{U}_{t,j}^l\right)
    \leq \sqrt{\sum_{l\in\mathcal{S}_t}\mathsf{diam}\left(\mathcal{U}_{t,j}^l\right)^2}
    \leq \sqrt{C} (D+2\eta K).
\end{align*}
Following a similar analysis in the proof of Lemma~\ref{lem::Noisy_ite_alg_SDPI}, we have
\begin{align}
    {\textnormal T}(\rW_T;\rZ_i^k)  \nonumber
    % &\leq {\textnormal T}(\{\rW_{T,M}^k\}_{k\in \mathcal{S}_T}; \rZ_i^k) \\
    % &\leq q^M \cdot {\textnormal T}(\rW_{T-1}; \rZ_i^k)\\
    % &\leq ...\\
    &\leq q^{M (T-t)} \cdot {\textnormal T}(\rW_{t}; \rZ_i^k) \nonumber\\
    &\leq q^{(M-j) + M (T-t)} \cdot {\textnormal T}(\{\rW_{t,j}^{l}\}_{l\in \mathcal{S}_t}; \rZ_i^k), \label{eq::FL_MI_ub_TD}
\end{align}
where the constant $q$ is defined as
\begin{align*}
    q \defined 1-2\bar{\Phi}\left(\frac{\sqrt{C}(D+2\eta K)}{2\eta} \right).
\end{align*}
Analogous to the proof of Lemma~\ref{lem::ub_MI_WT_Zi_general}, we have
\begin{align}
\label{eq::FL_MI_var_HWI}
    {\textnormal T}(\{\rW_{t,j}^{l}\}_{l\in \mathcal{S}_t}; \rZ_i^k)
    \leq \frac{1}{b} \EE{\|g(\rW_{t,j-1}^{k}, \rZ^k) - \be\|_2}
\end{align}
where $\be \defined \EE{g(\rW_{t,j-1}^{k}, \rZ^k)}$. Since the data point $\rZ_i^k$ is only used by the client $k\in \mathcal{S}_t$, the right-hand side of \eqref{eq::FL_MI_var_HWI} does not involve $\{\rW_{t,{j-1}}^{l}\}_{l\in \mathcal{S}_t\backslash \{k\}}$. Combining \eqref{eq::FL_MI_ub_TD}, \eqref{eq::FL_MI_var_HWI} with the ${\textnormal T}$-information bound in Lemma~\ref{lem::Bu20} yields the desired generalization bound for the $k$-th client.
\end{proof}

\subsection{Proof of Proposition~\ref{prop::gen_bound_SGLD}}
\label{append::gen_bound_SGLD}

We first present the following lemma whose proof follows by using the technique in Section~{\rm II}.~E of \citet{guo2005mutual}. 
\begin{lemma}
\label{lem::ext_MI_GC_ub}
Let $\rX$ be a random variable which is independent of $\rN \sim N(0, \mathbf{I}_d)$. Then for any $m>0$ and deterministic function $f$
\begin{align}
    I(f(\rX) + m\rN; \rX) 
    \leq \frac{1}{2m^2} \Var{f(\rX)}.
\end{align}
More generally, if $\rZ$ is another random variable which is independent of $\rN$, then for any fixed $\bz$
\begin{align}
\label{eq::cond_MI_ub_2moment}
    I(f(\rX) + m\rN; \rX\mid\rZ=\bz) 
    \leq \frac{1}{2m^2} \Var{f(\rX)\mid\rZ=\bz}.
\end{align}
\end{lemma}
\begin{proof}
By the property of mutual information \citep[see Theorem~2.3 in][]{polyanskiy2014lecture}, 
\begin{align}
\label{eq::MI_scale}
    I(f(\rX) + m\rN; \rX) 
    = I\left(\frac{f(\rX) - \be}{m} + \rN; \rX\right)
\end{align}
where $\be \defined \EE{f(\rX)}$. We denote 
\begin{align}
\label{eq::defn_g_fun}
    g(\bx) \defined \frac{f(\bx)-\be}{m}.
\end{align}
The golden formula \citep[see Theorem~3.3 in][for a proof]{polyanskiy2014lecture} yields
\begin{align}
    I\left(g(\rX) + \rN; \rX\right)
    &= \KL\left(P_{g(\rX)+\rN|\rX}\| P_{\rN}|P_{\rX}\right) - \KL\left(P_{g(\rX)+\rN}\|P_{\rN}\right) \nonumber\\
    &\leq \KL\left(P_{g(\rX)+\rN|\rX}\| P_{\rN}|P_{\rX}\right). \label{eq::proof_Gaussian_channel_MI_KL}
\end{align}
Furthermore, since $\rX$ and $\rN$ are independent, we have
\begin{align*}
    \KL\left(P_{g(\rX)+\rN|\rX=\bx}\| P_{\rN}\right)
    =\KL\left(P_{g(\bx)+\rN}\| P_{\rN}\right)
    = \frac{\|g(\bx)\|_2^2}{2},
\end{align*}
where the last step is due to the closed-form expression of the KL-divergence between two Gaussian distributions. Finally, by the definition of conditional divergence, we have
\begin{align}
\label{eq::proof_Gaussian_channel_KL_var}
    \KL\left(P_{g(\rX)+\rN|\rX}\| P_{\rN}|P_{\rX}\right)
    = \frac{1}{2}\EE{\|g(\rX)\|_2^2}
    = \frac{1}{2m^2}\Var{f(\rX)},
\end{align}
where the last step is due to the definition of $g$ in \eqref{eq::defn_g_fun}. Combining (\ref{eq::MI_scale}--\ref{eq::proof_Gaussian_channel_KL_var}) leads to the desired conclusion. Finally, it is straightforward to obtain \eqref{eq::cond_MI_ub_2moment} by conditioning on $\rZ=\bz$ and repeating our above derivations.
\end{proof}

Next, we present the second lemma which will be used for proving Proposition~\ref{prop::gen_bound_SGLD}.
\begin{lemma}
\label{lem::SGLD_general}
If the loss function $\ell(\bw, \rZ)$ is $\sigma$-sub-Gaussian under $\rZ\sim \mu$ for all $\bw \in \mathcal{W}$, the expected generalization gap of the SGLD algorithm can be upper bounded by
\begin{align*}
    \frac{\sqrt{2}\sigma}{2n} \sum_{j=1}^m \sqrt{\sum_{t\in \mathcal{T}_j} \beta_t \eta_t \cdot \Var{\nabla_{\bw} \hat{\ell}(\rW_{t-1}, \bar{\rZ}_j)}},
\end{align*}
where the set $\mathcal{T}_j$ contains the indices of iterations in which the mini-batch $\rS_j$ is used and the variance is over the randomness of $(\rW_{t-1},\bar{\rZ}_j)\sim P_{\rW_{t-1},\bar{\rZ}_j}$ with $\bar{\rZ}_j$ being any data point in the mini-batch $\rS_j$.
\end{lemma}

\begin{proof}
We denote $\rZ^{(k)} \defined (\rZ_1,\cdots,\rZ_{k})$ for $k\in [n]$ and $\rW^{(t)} \defined (\rW_1,\cdots,\rW_t)$ for $t\in [T]$. For simplicity, in what follows we only provide an upper bound for $I(\rW; \rZ_n)$. Since $\rW$ is a function of $\rW^{(T)} = (\rW_1,\cdots,\rW_T)$, the data processing inequality yields
\begin{align}
\label{eq::I_W_Zn_ub_WZn1Zn}
    I(\rW; \rZ_n) 
    \leq I(\rW^{(T)}; \rZ_n)
    \leq I(\rW^{(T)}, \rZ^{(n-1)}; \rZ_n).
\end{align}
By the chain rule,
\begin{align}
\label{eq::MI_chain_WT_Zn1_Zn}
    I(\rW^{(T)}, \rZ^{(n-1)}; \rZ_n)
    = I(\rW_T; \rZ_n \mid \rW^{(T-1)},\rZ^{(n-1)}) + I(\rW^{(T-1)},\rZ^{(n-1)}; \rZ_n).
\end{align}
Let $\bw=(\bw_1,\cdots,\bw_{T-1})$ and $\bz = (\bz_1,\cdots,\bz_{n-1})$ be any two vectors. If $\rZ_n$ is not used at the $T$-th iteration, without loss of generality we assume that the data points $\rZ_1,\cdots,\rZ_b$ are used in this iteration. Then
\begin{align}
    &I(\rW_T; \rZ_n \mid \rW^{(T-1)} = \bw, \rZ^{(n-1)}=\bz) \nonumber\\
    &= I\left(\bw_{T-1} - \frac{\eta_T}{b} \sum_{i = 1}^b \nabla_{\bw} \hat{\ell}(\bw_{T-1}, \bz_{i}) + \sqrt{\frac{2\eta_T}{\beta_T}} \rN_T; \rZ_n \mid \rW^{(T-1)} = \bw, \rZ^{(n-1)}=\bz\right) \nonumber\\
    &= I\left(\rN_T; \rZ_n \mid \rW^{(T-1)} = \bw, \rZ^{(n-1)}=\bz\right) \nonumber\\
    &=0.\label{eq::MI_0_Zn_not_use}
\end{align}
On the other hand, if $\rZ_n$ is used at the $T$-th iteration, without loss of generality we assume that the other $b-1$ data points which are also used in this iteration are $\rZ_1,\cdots,\rZ_{b-1}$. Then
\begin{align}
    &I(\rW_T; \rZ_n \mid \rW^{(T-1)} = \bw, \rZ^{(n-1)}=\bz) \nonumber\\
    &= I\left(\bw_{T-1} - \frac{\eta_T}{b} \left(\sum_{i=1}^{b-1} \nabla_{\bw} \hat{\ell}(\bw_{T-1}, \bz_{i}) + \nabla_{\bw} \hat{\ell}(\bw_{T-1}, \rZ_{n})\right) + \sqrt{\frac{2\eta_T}{\beta_T}} \rN_T; \rZ_n \mid \rW^{(T-1)} = \bw, \rZ^{(n-1)}=\bz\right) \nonumber\\
    &= I\left(- \frac{\eta_T}{b} \nabla_{\bw} \hat{\ell}(\bw_{T-1}, \rZ_{n}) + \sqrt{\frac{2\eta_T}{\beta_T}} \rN_T; \rZ_n \mid \rW^{(T-1)} = \bw, \rZ^{(n-1)}=\bz\right). \label{eq::MI_eq_exp_Zn_is_used}
\end{align}
By Lemma~\ref{lem::ext_MI_GC_ub}, we have
\begin{align}
    &I\left(- \frac{\eta_T}{b} \nabla_{\bw} \hat{\ell}(\bw_{T-1}, \rZ_{n}) + \sqrt{\frac{2\eta_T}{\beta_T}} \rN_T; \rZ_n \mid \rW^{(T-1)} = \bw, \rZ^{(n-1)}=\bz\right) \nonumber\\
    &\leq \frac{\beta_T \eta_T}{4b^2} \Var{\nabla_{\bw} \hat{\ell}(\bw_{T-1}, \rZ_{n})\mid \rW^{(T-1)} = \bw, \rZ^{(n-1)}=\bz}. \label{eq::MI_ub_var_Zn_is_used}
\end{align}
Substituting \eqref{eq::MI_ub_var_Zn_is_used} into \eqref{eq::MI_eq_exp_Zn_is_used} gives
\begin{align*}
    I(\rW_T; \rZ_n \mid \rW^{(T-1)} = \bw, \rZ^{(n-1)}=\bz)
    \leq \frac{\beta_T \eta_T}{4b^2} \Var{\nabla_{\bw} \hat{\ell}(\bw_{T-1}, \rZ_{n})\mid \rW^{(T-1)} = \bw, \rZ^{(n-1)}=\bz}.
\end{align*}
Taking expectation w.r.t. $(\rW^{(T-1)}, \rZ^{(n-1)})$ on both sides of the above inequality and using the law of total variance lead to 
\begin{align}
\label{eq::MI_ub_var_useZn}
    I(\rW_T; \rZ_n \mid \rW^{(T-1)}, \rZ^{(n-1)})
    \leq \frac{\beta_T \eta_T}{4b^2} \Var{\nabla_{\bw} \hat{\ell}(\rW_{T-1}, \rZ_{n})}.
\end{align}
To summarize, \eqref{eq::MI_0_Zn_not_use} and \eqref{eq::MI_ub_var_useZn} can be rewritten as
\begin{equation}
\label{eq::MI_WT_Zn_ub_var_0}
\begin{aligned}
    &I(\rW_T; \rZ_n \mid \rW^{(T-1)}, \rZ^{(n-1)})\\
    &\leq 
    \begin{cases}
    \frac{\beta_T \eta_T}{4b^2} \Var{\nabla_{\bw} \hat{\ell}(\rW_{T-1}, \rZ_{n})}\quad &\text{if } \rZ_n \text{ is used at the }T\text{-th iteration},\\
    0 &\text{otherwise.}
    \end{cases}
\end{aligned}
\end{equation}
Assume that the data point $\rZ_n$ belongs to the $j$-th mini-batch $\rS_j$. Now substituting \eqref{eq::MI_WT_Zn_ub_var_0} into \eqref{eq::MI_chain_WT_Zn1_Zn} and doing this procedure recursively lead to
\begin{align*}
    I(\rW^{(T)}, \rZ^{(n-1)}; \rZ_n)
    \leq \sum_{t\in \mathcal{T}_j} \frac{\beta_t \eta_t}{4b^2} \Var{\nabla_{\bw} \hat{\ell}(\rW_{t-1}, \rZ_{n})},
\end{align*}
where the set $\mathcal{T}_j$ contains the indices of iterations in which the mini-batch $\rS_j$ is used. Hence, this upper bound along with \eqref{eq::I_W_Zn_ub_WZn1Zn} naturally gives
\begin{align}
\label{eq::MI_ub_W_Zn_sum_Var}
    I(\rW;\rZ_n) \leq \sum_{t\in \mathcal{T}_j} \frac{\beta_t \eta_t}{4b^2} \Var{\nabla_{\bw} \hat{\ell}(\rW_{t-1}, \rZ_{n})}.
\end{align}
By symmetry, for any data point in $\rS_j$ besides $\rZ_n$, the mutual information between $\rW$ and this data point can be upper bound by the right-hand side of \eqref{eq::MI_ub_W_Zn_sum_Var} as well. Finally, recall that Lemma~\ref{lem::Bu20} provides an upper bound for the expected generalization gap:
\begin{align}
    \frac{\sqrt{2}\sigma}{n}\sum_{i=1}^n \sqrt{I(\rW_T; \rZ_i)}
    = \frac{\sqrt{2}\sigma}{n} \sum_{j=1}^m \sum_{\rZ \in \rS_j} \sqrt{I(\rW_T; \rZ)}.
\end{align}
By substituting \eqref{eq::MI_ub_W_Zn_sum_Var} into the above expression, we know the expected generalization gap can be further upper bounded by
\begin{align*}
    \frac{\sqrt{2}\sigma}{2n} \sum_{j=1}^m \sqrt{\sum_{t\in \mathcal{T}_j} \beta_t \eta_t \cdot \Var{\nabla_{\bw} \hat{\ell}(\rW_{t-1}, \bar{\rZ}_j)}},
\end{align*}
where $\bar{\rZ}_j$ is any data point in the mini-batch $\rS_j$.
\end{proof}

Finally, we are in a position to prove Proposition~\ref{prop::gen_bound_SGLD}. 
\begin{proof}
Consider a new loss function and the gradient of a new surrogate loss:
\begin{align*}
    \ell(\bw, \rS_j) 
    \defined \frac{1}{b} \sum_{\rZ \in \rS_j} \ell(\bw, \rZ),\quad
    \nabla_{\bw} \hat{\ell}(\bw, \rS_{j}) 
    \defined \frac{1}{b} \sum_{\rZ\in\rS_j} \nabla_{\bw} \hat{\ell}(\bw, \rZ).
\end{align*}
Then $\ell(\bw, \rS_j)$ is $\sigma/\sqrt{b}$-sub-Gaussian under $\rS_j\sim \mu^{\otimes b}$ for all $\bw\in\mathcal{W}$. We view each mini-batch $\rS_j$ as a data point and view $\ell(\bw, \rS_j)$ as a new loss function. By using Lemma~\ref{lem::SGLD_general}, we obtain:
\begin{align}
\label{eq::SGLD_minibatch_M}
    \left|\EE{L_\mu(\rW)-L_{\rS}(\rW)}\right|
    \leq \frac{\sqrt{2}\sigma}{2m\sqrt{b}} \sum_{j=1}^m \sqrt{\sum_{t\in \mathcal{T}_j} \beta_t \eta_t \cdot \Var{\nabla_{\bw} \hat{\ell}(\rW_{t-1}, \rS_j)}}.
\end{align}
Since the data set contains $n$ data points and is divided into $m$ disjoint mini-batches with size $b$, we have $n = mb$. Substituting this into \eqref{eq::SGLD_minibatch_M} leads to the desired conclusion.
\end{proof}

\subsection{Proof of Corollary~\ref{cor::SGLD_gen_bound_trajectory}}
\label{append::SGLD_gen_bound_trajectory}

\begin{proof}
The Minkowski inequality implies that for any non-negative $x_i$, the inequality $\sqrt{\sum_i x_i}\leq \sum_i \sqrt{x_i}$ holds. Therefore, we can further upper bound the generalization bound in Lemma~\ref{lem::SGLD_general} by
\begin{align*}
    \frac{\sqrt{2}\sigma}{2n} \sum_{j=1}^m \sum_{t\in \mathcal{T}_j} \sqrt{\beta_t \eta_t \cdot \Var{\nabla_{\bw} \hat{\ell}(\rW_{t-1}, \bar{\rZ}_j)}}
    = \frac{\sqrt{2}\sigma}{2n} \sum_{t=1}^T \sqrt{\beta_t \eta_t \cdot \Var{\nabla_{\bw} \hat{\ell}(\rW_{t-1}, \rZ^{\dagger}_{t})}}.
\end{align*}
Alternatively, by Jensen's inequality and $n=m b$, we can further upper bound the generalization bound in Lemma~\ref{lem::SGLD_general} by
\begin{align*}
    \frac{\sqrt{2}\sigma}{2} \sqrt{\frac{1}{bn}\sum_{t=1}^T\beta_t \eta_t \cdot \Var{\nabla_{\bw} \hat{\ell}(\rW_{t-1}, \rZ^{\dagger}_{t})}}.
\end{align*}
\end{proof}

\section{Supporting Experimental Results}
\label{append::experiments}

Recall that our generalization bound in Proposition~\ref{prop::gen_bound_SGLD} involves the variance of gradients. To estimate this quantity from data, we repeat our experiments 4 times and record the batch gradient at each iteration. This batch gradient is the one used for updating the parameters in the SGLD algorithm so it does not require any additional computations. Then we estimate the variance of gradients by using the population variance of the recorded batch gradients. Finally, we repeat the above procedure 4 times for computing the standard deviation, leading to e.g., the shaded areas in Figure~\ref{Fig::label_corruption}.
We provide experimental details in Table~\ref{table:label_corruption_MNIST_detail}~and~\ref{table:cnn_details} and code in \hyperlink{https://github.com/yih117/Analyzing-the-Generalization-Capability-of-SGLD-Using-Properties-of-Gaussian-Channels}{https://github.com/yih117/Analyzing-the-Generalization-Capability-of-SGLD-Using-Properties-of-Gaussian-Channels} for reproducing our experiments.

\begin{table}
\small\centering
\resizebox{0.8\textwidth}{!}{
\renewcommand{\arraystretch}{1.5}
\begin{tabular}{ll}
\toprule
Parameter & Details \\  
\toprule
Data set & MNIST\\
\midrule
Number of training data & 5000\\
\midrule
Batch size & 500\\
\midrule
Learning rate & Initialization = 0.03, decay rate = 0.96, decay steps=2000 \\
\midrule
Inverse temperature & $\beta_t = 10^{6} / (2 \eta_t)$ \\
\midrule
Architecture & MLP with ReLU activation\\
\midrule
Depth & 3 layers\\
\midrule
Width & 64 hidden units\\
\midrule
Objective function & Cross-entropy loss\\
\midrule
Loss function & 0-1 loss\\
\bottomrule 
\end{tabular}
}
\caption{\small{Experiment details of Figure~\ref{Fig::label_corruption} and \ref{Fig::Hidden_Units} on the MNIST data set. The network width is varying among $\{16,32,64,128,256\}$ hidden units for Figure~\ref{Fig::Hidden_Units}.}}
\label{table:label_corruption_MNIST_detail}
\end{table}

\begin{table}
\small\centering
\resizebox{0.8\textwidth}{!}{
\renewcommand{\arraystretch}{1.5}
\begin{tabular}{ll}
\toprule
Parameter & Details \\  
\toprule
Data set & CIFAR-10 and SVHN\\
\midrule
Number of training data & 5000\\
\midrule
Batch size & 500\\
\midrule
Learning rate  & Initialization = 0.03, decay rate = 0.96, decay steps = 2000\\
\midrule
Inverse temperature & $\beta_t = 10^6 / (2\eta_t)$ \\
\midrule
\
Architecture 
& $\mathsf{conv}(5, 32)$ $\mathsf{pool}(2)$ $\mathsf{conv}(5, 32)$ $\mathsf{pool}(2)$ $\mathsf{fc}(120)$ $\mathsf{fc}(84)$ $\mathsf{fc}(10)$\\
\midrule
Objective function & Cross-entropy loss\\
\midrule
Loss function & 0-1 loss\\
\bottomrule 
\end{tabular}
}
\caption{\small{Experiment details of Figure~\ref{Fig::label_corruption} and \ref{Fig::Hidden_Units} on the CIFAR-10 and SVHN data sets. Here $\mathsf{conv}(k, w)$ is a $k \times k$ convolutional layer with $w$ filters; $\mathsf{pool}(k)$ is a $k \times k$ max pooling layer; and $\mathsf{fc}(k)$ is a fully connected layer with $k$ units. The convolutional layers and the fully connected layers all use ReLU activation function. The network width (i.e., number of filters in CNN) is varying among $\{8,16,32,64,128\}$ for Figure~\ref{Fig::Hidden_Units}.}}
\label{table:cnn_details}
\end{table}

\clearpage
\bibliography{references}

\end{document}